\documentclass{article}

\usepackage[utf8]{inputenc} 
\usepackage[T1]{fontenc}    
\usepackage{hyperref}       
\usepackage{url}            
\usepackage{booktabs}       
\usepackage{amsfonts}       
\usepackage{nicefrac}       
\usepackage{microtype}      
\usepackage{xcolor}         
\usepackage{subcaption}

\usepackage{algorithm}

\usepackage{algpseudocode}

\usepackage{threeparttable}

\newcommand{\1}{\mathbb{I}}

\usepackage{multirow}

\usepackage{mathtools}
\usepackage{amsthm}
\usepackage{tabularx}

\newcolumntype{Y}{>{\centering\arraybackslash}X}

\usepackage[capitalize,noabbrev]{cleveref}

\usepackage[printonlyused]{acronym}
\acrodef{SCP}{Split Conformal Prediction}
\acrodef{APS}{Adaptive Prediction Sets}
\acrodef{TSCP}{Transductive Split Conformal Prediction}

\usepackage{amsmath,amsfonts,bm}

\def\eqref#1{eq.~\ref{#1}}

\def\1{\bm{1}}

\def\vv{{\bm{v}}}

\def\vx{{\bm{x}}}

\def\mX{{\bm{X}}}
\def\mY{{\bm{Y}}}
\def\mZ{{\bm{Z}}}

\DeclareMathAlphabet{\mathsfit}{\encodingdefault}{\sfdefault}{m}{sl}
\SetMathAlphabet{\mathsfit}{bold}{\encodingdefault}{\sfdefault}{bx}{n}

\def\gC{{\mathcal{C}}}

\def\gX{{\mathcal{X}}}
\def\gY{{\mathcal{Y}}}

\usepackage{amsthm}
\usepackage{amssymb}
\newtheorem{theorem}{Theorem}[section]

\newtheorem{lemma}[theorem]{Lemma}
\theoremstyle{definition}
\newtheorem{definition}[theorem]{Definition}

\theoremstyle{remark}
\newtheorem{remark}[theorem]{Remark}

\def\P{\mathbb{P}}

\newcommand*{\paran}[1]{\left(#1\right)}
\newcommand*{\prob}[1]{\mathbb{P}}

\newcommand*{\card}[1]{\left|#1\right|}

\newcommand{\E}{\mathbb{E}}

\newcommand{\R}{\mathbb{R}}

\newcommand{\quant}{\text{Quantile}}

\usepackage{empheq}
\usepackage{fancybox}
\usepackage{tcolorbox}

\usepackage{selectp}

\newcommand{\Q}{\mathrm{Q}}

\usepackage[a4paper,bindingoffset=0.2in,%
            left=0.9in,right=0.9in,top=1in,bottom=1in,%
            footskip=.25in]{geometry}
\hypersetup{
    colorlinks=true,
    linkcolor=blue,
    citecolor=blue,
    filecolor=magenta,      
    urlcolor=cyan,
}

\usepackage{natbib}

\usepackage{todonotes}[disable]

\newcommand*{\SKIP}[1]{}

\title{Fundamental bounds on efficiency-confidence trade-off for transductive conformal prediction}

\date{}
\author{
Arash Behboodi, Alvaro H.C. Correia, Fabio Valerio Massoli, Christos Louizos \\{Qualcomm AI research}\thanks{Qualcomm AI Research is an initiative of Qualcomm Technologies, Inc.} 
}
 
\begin{document}
\maketitle

\begin{abstract}
 Transductive conformal prediction addresses the simultaneous prediction for multiple data points. Given a desired confidence level, the objective is to construct a prediction set that includes the true outcomes with the prescribed confidence. We demonstrate a fundamental trade-off between confidence and efficiency in transductive methods, where efficiency is measured by the size of the prediction sets. Specifically, we derive a strict finite-sample bound showing that any non-trivial confidence level leads to exponential growth in prediction set size for data with inherent uncertainty. The exponent scales linearly with the number of samples and is proportional to the conditional entropy of the data. Additionally, the bound includes a second-order term, dispersion, defined as the variance of the log conditional probability distribution. We show that the transductive methods based on the approximate conditional distribution can approach this bound. Inspired by this setup, we introduce a practical transductive prediction algorithm that surpasses Bonferroni methods.
\end{abstract}

\section{Introduction}
Modern decision systems often need to \emph{predict many outcomes at once} and act on the \emph{joint result}. Examples include certifying components in a quality-control batch, screening biological samples for pathogens, or approving software changes before release.
In such settings, the cost of even a single error can be high, making \emph{distribution-free guarantees} on the \emph{entire vector of predictions} essential.

Conformal prediction (CP) \citep{vovk_algorithmic_2022} offers a principled framework for constructing prediction sets with finite-sample, distribution-free coverage guarantees under minimal assumptions. Typically, CP methods operate on individual input–output pairs, where each input $X$ is associated with a label $Y$. However, many real-world systems require \emph{joint guarantees} across multiple predictions, motivating the study of \emph{transductive conformal prediction (TCP)}. TCP constructs a joint prediction set for a batch of test inputs $X_1, \dots, X_n$, ensuring that the corresponding label vector $Y_1, \dots, Y_n$ lies within the set with a prescribed confidence level (e.g., 95\%).

TCP offers stronger guarantees, better than combining individual conformal prediction. The main theme of this paper is the efficiency of TCP predictions, for example, measured in terms of the prediction set size. How small can such joint sets be, on average, while still guaranteeing coverage? How can one design efficient TCP methods? These questions are not merely practical, they probe the limits of uncertainty quantification in multi-output prediction. Our paper addresses these questions and proposes a framework which sets the target for the efficiency of practical TCP methods. 
Our contributions are as follows. We derive \textbf{fundamental bounds} on  the efficiency-confidence trade-off. We prove that for any non-trivial confidence level, the expected size of any valid joint prediction set must grow \emph{exponentially} with the number of test points. The growth rate is governed by the conditional entropy $H(Y|X)$ and a second-order term we call \emph{dispersion}, which captures the variance of the log-conditional probabilities. We show that this bound is \emph{tight} by constructing an idealized predictor (with oracle access to $P(Y|X)$) that matches the first and second-order terms. We show that approximate conditional distribution can be used to approach this bound, and the additional efficiency penalty scales gracefully with the approximation error of the distribution. Finally, we introduce a transductive split conformal prediction scheme for a practical and efficient prediction.

These results hold under minimal assumptions: they apply to any conformity score, extend to a larger class of efficiency metrics beyond prediction set size, and are validated by experiments showing their relevance in finite-sample regimes, highlighting inefficiencies in existing transductive methods. This result can be seen as a framework for designing and guiding efficient predictors akin to Shannon theory for code design. For synthetic settings with computable entropy and dispersion terms, the research works can focus on closing the gap with the fundamental limits derived in this work. As an example, we provide a new transductive prediction algorithm inspired by our theory and show its efficiency.

\section{Transductive Conformal Predictors}

To prepare for our main results, this section contrasts standard conformal prediction (CP), which offers marginal coverage for individual predictions, with its transductive extension (TCP) that provides joint guarantees across a batch.

\paragraph{Notation.} In this paper, the random variables are denoted by capital letters $X_1, X_2,...$ and their realization by $x_1,x_2,\dots$, and the vectors and matrices are denoted by bold letters as $\mX,\vx$. $X_i^j$ denotes the tuple $(X_i,\dots, X_j)$. We use $P(Y|X)$ to denote the conditional distribution of labels $Y$ given samples $X$. We use $P$ as well to denote the distribution over $X$ and $Y$. The logarithms are all assumed to be natural logarithms, unless otherwise stated. The entropy $H(X)$ is defined as $\E[-\log P(X)]$, the conditional entropy defined as $H(Y|X):=\E[-\log P(Y|X)]$. The Kullback-Leibler divergence is defined as $D(Q\Vert P):= \E_Q[\log \mathrm{d}Q/\mathrm{d}P]$.  $\Q(\cdot)$ is the Gaussian Q-function defined as $\mathrm{Q}(t):=\P(X > t)$ for $X$ a standard normal distribution.

\paragraph{From Standard to Transductive Conformal Prediction (TCP).}
Standard conformal prediction (CP) constructs a \emph{per-input} prediction set that contains the true label with probability at least $1-\alpha$, under exchangeability. Formally, consider a sequence of labeled examples $Z_{1}^{m}=((X_i,Y_i):i\in[m])$, where $X_i\in\gX, Y_i\in\gY$ with $M$ distinct classes, i.e., $|\gY|=M$, and a test sample $X_{m+1}$ with unknown true label $Y_{m+1}$. Standard CP produces a set $\Gamma^{\alpha}(X_{m+1})$ that satisfy \emph{marginal coverage}: $\P(Y_{m+1} \notin \Gamma^{\alpha}(X_{m+1})) \leq \alpha.$ These sets are obtained by thresholding p-values: for each input, all labels with p-value above $\alpha$ are included. A popular variant, split CP \citep{Papadopoulos2002-hw,lei2018distribution}, computes these p-values using a pretrained predictor and a separate calibration set. More generally, CP can be formalized through \emph{transducers}, which map sequences of labeled examples to p-values in $[0,1]$, providing a unified view of conformal methods.

While the \emph{marginal} guarantee of standard CP is often sufficient for isolated decisions, many applications require system-level guarantees, such as maintaining a global missed-detection constraint in autonomous driving or ensuring consistency in ranking tasks \citep{Fermanian2025-transduc-ranking}. In these settings, a single error can invalidate the entire outcome, motivating \emph{joint} guarantees. Transductive conformal prediction (TCP) addresses this by constructing a joint prediction set for the whole test batch  \citep{Vovk2013-transductive-pred,vovk_algorithmic_2022}.
Given $Z_1^m$ and a batch of test samples $X_{m+1}^{m+n} = (X_{m+1},\dots,X_{m+n})$, a (transductive) confidence predictor outputs a set of candidate label vectors $(\mY_{m+1},\dots,\mY_{m+n})$ such that the error probability of the predictor $P_e$ is bounded
\begin{equation} \label{eq:tcp_coverage}
    P_e =\P\paran{Y_{m+1}^{m+n}\notin \Gamma^{\alpha}(Z_{1}^{m},X_{m+1}^{m+n})} \leq \alpha.
\end{equation}
If the predictor satisfies the significance level $\alpha$, we say it has confidence $1-\alpha$. Some works instead use the \emph{False Coverage Proportion (FCP)} as the error \citep{Fermanian2025-transduc-ranking}, which measures the average per-sample error rather than the joint error over the entire test set. This is a more relaxed criterion than the one adopted here and in \citep{vovk_algorithmic_2022} (see Appendix~\ref{supp:comparison} for details).

Operationally, TCP extends the conformal principle from single examples to sequences: instead of computing p-values for individual labels, we compute them for entire candidate label sequences using \emph{transductive conformity scores}. These scores assess how well a proposed joint labeling fits the observed data and the test batch. Thresholding these p-values yields a joint confidence set that guarantees coverage for all test points simultaneously. A common baseline is to aggregate per-sample p-values via Bonferroni:
build $\Gamma^{\alpha/n}(X_{m+i})$ for each test point and take the Cartesian product $\prod_{i=1}^n \Gamma^{\alpha/n}(X_{m+i})$, which satisfies~\eqref{eq:tcp_coverage} but can be inefficient as $n$ grows (cf.~our experiments).

\paragraph{Efficiency vs.\ Confidence.}
While~\eqref{eq:tcp_coverage} guarantees \emph{confidence}, practitioners also care about \emph{efficiency}: how large the joint prediction set is on average. These two objectives are inherently in tension, and understanding this trade-off is among the main objectives of this work. To that end, we measure efficiency by the cardinality of $\Gamma^{\alpha}(Z_{1}^{m},X_{m+1}^{m+n})$, though our results extend to some other notions of efficiency (see Appendix~\ref{sec:eff-conf-general}). Of particular interest is the \emph{efficiency rate}, which captures the exponential growth of the expected prediction set size as the number of test samples $n$ increases.

\begin{definition}
The \emph{efficiency rates} of a transductive conformal predictor are
\begin{align*}
\gamma_{n,m} &:=\frac{1}{n}\log\E\big|\Gamma^{\alpha}(Z_{1}^{m},X_{m+1}^{m+n})\big|,\\
\gamma_m^+&:=\limsup_{n\to\infty}\gamma_{n,m},\;
\gamma_m^-:=\liminf_{n\to\infty}\gamma_{n,m}.    
\end{align*}
When these limits coincide, we denote the outcome by $\gamma_m$.
\end{definition}

\paragraph{Research Questions and Road Map.} Building on the formalism above, we revisit the interpretation in \citet{Correia2024-zq}, where split CP is viewed through the lens of list decoding \citep{wozencraft_list_1958}. In this setting, the model output is treated as a noisy observation of the ground truth label, which enabled the authors to establish information-theoretic inequalities linking the efficiency of conformal prediction, as measured by the expected size of the prediction set, to the conditional entropy $H(Y|X)$ of the labeling distribution. 

However, some fundamental questions remain open. The \textbf{efficiency-confidence trade-off} characterized in \citet{Correia2024-zq} in terms of $H(Y|X)$ and the logarithm of the expected prediction set size was not tight. This raises the question: \emph{Can we derive tighter bounds on the efficiency-confidence trade-off, and characterize conditions under which these bounds are achievable?}

In the rest of this paper, we tackle these challenges. We establish fundamental bounds on the efficiency-confidence trade-off in the general transductive setting, capturing both asymptotic and finite-sample regimes and revealing a phase transition governed by the conditional entropy. Guided by the theoretical derivation, we introduce a practical transductive prediction algorithm that extends split conformal prediction to the transductive case.

\section{Fundamental Bounds on Efficiency-Confidence Trade-off}
\label{sec:eff_cof_trade_off}
\paragraph{Efficiency-Confidence Trade-off Theorem.}
In \citet{Correia2024-zq}, the authors derived new information-theoretic bounds that connected conformal prediction to list decoding. The bounds involved terms related to the efficiency of conformal prediction and the conditional entropy or KL-divergence terms and leveraged Fano's inequality and data processing inequality. In this section, we derive new bounds that can generally lead to tighter bounds. The proofs are deferred to Appendix \ref{sec:proof_off_cof_trade_off_section}. 

Consider the case where the error probability $P_e$ is not exceeding $\alpha$, that is
$\P\paran{Y_{m+1}^{m+n}\notin \Gamma^{\alpha}(Z_{1}^{m},X_{m+1}^{m+n})} \leq \alpha.$ We have the following result. 
\begin{theorem}[Efficiency-Confidence Trade-off Theorem]
    Consider a transductive conformal predictor $\Gamma^{\alpha}(Z_{1}^{m},X_{m+1}^{m+n})$ given a labeled dataset  $Z_{1}^{m}$ and test samples $X_{m+1}^{m+n}$ with unknown labels $Y_{m+1}^{m+n}$. If the predictor has the confidence $1-\alpha$, then for any $\beta\in (0,1)$, we have:
    \begin{equation*}
         \P( P(Y_{m+1}^{m+n}|X_{m+1}^{m+n}) \leq \beta )  \leq \alpha + \beta \E(|\Gamma^{\alpha}(Z_{1}^{m},X_{m+1}^{m+n})|) 
    \end{equation*}
\label{thm:verdu-han-transductive}
\vspace{-5mm}
\end{theorem}
Theorem \ref{thm:verdu-han-transductive} relates the inherent uncertainty about the labels, given by $P(Y_{m+1}^{m+n}|X_{m+1}^{m+n})$, to the expected prediction set size. 
When the inherent uncertainty is high, the conditional probability is more spread, and the probability $\P( P(Y_{m+1}^{m+n}|X_{m+1}^{m+n}) \leq \beta ) $  is high for smaller $\beta$, leading to larger lower bound on the prediction set size. 

\begin{remark}[\textbf{Other efficiency measures}]
    We have focused on the prediction set size as the notion of efficiency. It is possible to generalize this to other measures of efficiency on the prediction set. We show these results in Appendix \ref{sec:eff-conf-general}, in particular Theorem \ref{thm:verdu-han-withQ} as well as some examples of efficiency measures such as $S$-criterion and $N$-criterion \cite{vovk2016criteria}.

\end{remark}

\paragraph{Asymptotic phase transition of the expected prediction set size.}
As our first result, we show that, asymptotically, the expected prediction set size blows up exponentially for any nontrivial confidence. In other words, if the efficiency rate, i.e. the growth rate of $\E(|\Gamma^{\alpha}(Z_{1}^{m},X_{m+1}^{m+n})|) $, is slower that the conditional entropy $H(Y|X)$, the error converges to one. The following theorem summarizes the result.

\begin{theorem}
Consider a transductive conformal predictor $\Gamma^{\alpha}(Z_{1}^{m},X_{m+1}^{m+n})$ with confidence level $1-\alpha_n$ for $n$ test samples. Then, we have:
\begin{enumerate}
    \item If the asymptotic confidence is non-trivial, i.e., $\liminf_{n\to\infty}(1-\alpha_n)>0$, the efficiency rate satisfies:
    \[
    \gamma^{-}_m \geq H(Y|X).
    \]
    \item If $\gamma^{+}_m < H(Y|X)$, then the confidence vanishes asymptotically to zero:
    \[
    \lim_{n\to\infty}(1-\alpha_n)=0.
    \]
\end{enumerate}
\label{thm:fundamental_asymp_transductive}
\end{theorem}

\begin{remark}
The theorem states a fundamental asymptotic trade-off between the confidence level and the prediction set size. Roughly, the prediction set size needs to grow exponentially at least as $e^{nH(Y|X)}$ to avoid a non-trivial confidence level. Another insight is that asymptotically, there is a phase transition at the efficiency rate $H(Y|X)$ below which it is impossible to get non-trivial confidence. The result does not indicate anything regarding the impact of the asymptotic confidence level on the set size. Indeed, it can be seen that $\liminf_{n\to\infty}\frac{1}{n}\log(1-\alpha_n)=0$ for any non-trivial $\alpha_n$. In other words, it seems that it suffices to have $\gamma_m^{-}\geq H(Y|X)$ to get any non-trivial confidence \textit{asymptotically}. In the case of classical conformal prediction, where the prediction set of each sample is predicted independently, the result means that the expected prediction set size is greater than or equal to $e^{H(Y|X)}$, similarly reported in \citet{Correia2024-zq}.
\end{remark}

\paragraph{Efficiency rate in finite sample setting.}

We can derive a non-asymptotic bound for the efficiency-confidence trade-off using the growth rate of the average prediction set size. 
\begin{theorem}
    For a transductive conformal predictor with the confidence level $1-\alpha$, consider the efficiency rate defined as:
    \[
    \gamma_{n,m}:=\frac{1}{n}\log\E|\Gamma^{\alpha}(Z_{1}^{m},X_{m+1}^{m+n})|,
    \]
    which is the growth exponent of the prediction set size. Then for any $n$, we have:
    \begin{equation*}
 \log\Delta + nH(Y|X) +
 \sqrt{n}\sigma
  \Q^{-1}\paran{\alpha + \frac{\rho}{\sqrt{n}\sigma^3} + \Delta}  \leq { n\gamma_{n,m}}
\end{equation*}
if 
$\alpha + \frac{\rho}{\sqrt{n}\sigma^3} + \Delta\in[0,1]$ where $\Delta>0$ and  $\Q(\cdot)$ is the Gaussian Q-function, and:
\begin{align*}
    \sigma&:=\paran{\mathrm{Var}\paran{\log P(Y|X)}}^{1/2}\\
    &=\paran{\E\paran{\log P(Y|X) + H(Y|X)}^2}^{1/2}\\
    \rho &:=\E\paran{\card{\log P(Y|X) + H(Y|X)}^3}.
\end{align*}
\label{thm:nonasymptotic_transductive}
\vspace{-5mm}
\end{theorem}
The non-asymptotic results leverage the Berry-Esseen central limit theorem to characterize the sum $\sum_{i=1}^n\log\P(Y_{m+i}|X_{m+i})$ in Theorem \ref{thm:verdu-han-transductive}. We call the term $\sigma$, \textit{the dispersion} following a similar name used in finite block length analysis of Shannon capacity \citep{polyanskiy_channel_2010,Strassen1962-zp}.  Note that these bounds do not assume anything about the underlying predictor, and therefore do not show any dependence on the number of training samples $m$. The underlying method might as well have access to the underlying distributions $P(Y|X)$. To use the above bound, we provide an approximation by ignoring some constant terms that diminish with larger $n$. The approximate bound can be easily computed and is given as follows (see Appendix \ref{proof:thm:nonasymptotic_transductive} for the derivation):
\begin{empheq}[box=\Ovalbox]{align}
 {  n\gamma_{n,m}\geq nH(Y|X)+
 \sqrt{n}\sigma
  \Q^{-1}\paran{\alpha}  -\frac{\log n }{2} + O(1).}
  \label{eq:final_approximation}
\end{empheq}
\paragraph{On achievability of the bounds.} In this part, we argue that the provided bounds are achievable. 
Suppose that we know the underlying probability distribution $P(Y|X)$, which corresponds to the idealized setting in \citet{vovk_algorithmic_2022}. Upon receiving test samples $X_{1},\dots,X_{n}$, we can construct the confidence sets as follows:
\[
\Gamma^{\alpha}(X_1^n):= \left\{(y_1,\dots,y_n): \prod_{i=1}^n P(y_i|X_i) \geq \beta \right\}.
\]
Similarly, the efficiency rate is defined as $\gamma_n := \frac{1}{n}\log \E[|\Gamma^{\alpha}(X_{1}^{n})|]$. 
We show that with proper choice of $\beta$ we can achieve the lower bound, ignoring the logarithmic terms, at a given significance level $\alpha$. The definitions of $\rho$ and the dispersion $\sigma$ are similar to those in Theorem \ref{thm:nonasymptotic_transductive}.
\begin{theorem}
For the confidence set $\Gamma^{\alpha}(X_1^n)$ defined above, and for $\alpha\geq \rho/\sqrt{n}\sigma^3$, there is a choice of $\beta$  that achieves the confidence $1-\alpha$ at the efficiency rate $\gamma_n$ satisfying:
\[\label{thm:achievable}
n\gamma_n \leq n H(Y|X) + 
\sqrt{n}\sigma  \mathrm{Q}^{-1}(\alpha) + O(1).
\]
\end{theorem}
As it can be seen, knowing the underlying probability distribution, the prediction set size can be bounded, and therefore, the achievability bound matches the first and second order term in the converse bound. For the proof and more details see Appendix \ref{proof:supp:achievability}.

Although it is reassuring that these bounds are achievable in a idealized setting,  the conditional distribution $P(Y|X)$ is in general unknown. Nonetheless, this result provides a guideline on how to build prediction sets: \textit{approximate the conditional distribution and use it for building the prediction set}. In the context of classical conformal prediction, the authors of \citet{sadinle2019least} used k-nearest neighbors, local polynomial estimator and regularized multinomial logistic regression to approximate the conditional distributions and use it for prediction. See Appendix \ref{app:optimal_conf_predictor} for prior works on optimal confidence predictors.

A natural question is how the efficiency rate changes when an approximate conditional distribution $Q$ is used to build the prediction sets following the guideline: $
\Gamma_Q^{\alpha}(x_{1}^{n}):=\{y_{1}^{n}: {Q(y_{1}^{n}|x_{1}^{n})} \geq \beta \}.$ 
The following theorem characterizes the efficiency rate in terms of different divergences between $P(Y|X)$ and $Q(Y|X)$. 

\begin{theorem}
Consider the confidence set $\Gamma_Q^{\alpha}(X_1^n)$ defined above. Then, there is a choice of $\beta$  that achieves the confidence $1-\alpha$ at the efficiency rate $\gamma_n$ satisfying:
\begin{align*}
n\gamma_n \leq & n H(Y|X) + n \E[D(P(\cdot|X)\Vert Q(\cdot|X))] +
\sqrt{n}\sigma  \mathrm{Q}^{-1}(\alpha) + O(\log n),
\end{align*}

with $ \sigma  =\paran{\mathrm{Var}\paran{\log Q(Y|X)}}^{1/2}$ and $\rho=\E\paran{\card{\log Q(Y|X) -\mu}^3}$, subject to defined 
$\alpha\geq \rho/\sqrt{n}\sigma^3$. All expectations are with respect to the data distribution $P$.
\label{thm:approx_achievable}
\end{theorem}
Therefore, the key penalty for the distribution mismatch is the expected KL-divergence term $\E[D(P(\cdot|X)\Vert Q(\cdot|X))]$. See Appendix \ref{app:approx_dist} for the proof and more details. 

\paragraph{How to use the theory.} 
Our theoretical bounds can be used as a guideline to evaluate the efficiency of transductive methods in synthetic settings. In such cases, where  the conditional probability is known, and the bounds are computable, we can evaluate how close the existing transductive methods are to the fundamental bound. This is in analogy with research in the field of information theory where binary symmetric channels, erasure channels and AWGN channels are used as benchmark for designing error correction codes. We use this vision in our numerical result section, where we evaluated the Bonferroni method and showed its shortcomings.  Note that the information-theoretic terms $H(Y|X)$ and $\sigma$ are in general unknown. If we approximate the conditional distribution and use it both for approximating these terms and constructing the prediction sets, then based on the above discussion, the lower and upper bounds will coincide and provide no additional insight.

Finally, the results of Theorem \ref{thm:achievable} and \ref{thm:approx_achievable} highlight the centrality of (approximate) conditional distribution in constructing the optimal prediction sets. Therefore, there is a close connection between the conditional distribution and the (non)conformity score in conformal prediction. We explore this connection in the next section, where approximate model probabilities are used as a conformity score. 

\section{ Transductive Split Conformal Prediction}

Fully transductive conformal prediction given by Vovk \citet{Vovk2013-transductive-pred} is computationally prohibitive. Similar to the classical full CP, it requires multiple \textit{retraining} of the model, for which the complexity grows exponentially with the number of test points. To avoid the exponential growth penalty, Vovk suggested to use  Bonferroni correction \citet{Vovk2013-transductive-pred}. Such a predictor constructs a $p$-value for each of the $n$ test samples separately and aggregates these $p$-values, $p_1,\dots,p_n$ using Bonferroni formula:
\[
p := \min(np_1,\dots,np_n,1).
\]
All the labels $(v_1,\dots,v_n)$ with $p$-values larger than the inconfidence $\alpha$ are included in the prediction set. Note that $p$-values can be obtained with a nonconformity score \cite{vovk_algorithmic_2022}. Therefore, Bonferroni prediction amounts to running standard conformal prediction for each test sample with the modified confidence level $\alpha/n$ and getting the set product of the predicted sets (see Appendix \ref{supp:tcp_vovk} for more details). We will use the Bonferroni-corrected version of \ac{SCP} \citet{Papadopoulos2002-hw} and \ac{APS} \cite{Romano2020-adaptive} for the experiments. As we will see, Bonferroni method is very inefficient for larger $n$.  

\paragraph{Constructing conformity metrics from model scores.} We saw in Theorem \ref{thm:achievable} that the optimal prediction set can be constructed using the true conditional distribution and even with the approximate version of it. The main idea was to select a threshold $\beta$ based on the confidence $\alpha$ and then include all labels $(y_1,\dots,y_n)$ with the product probability $\prod_{i=1}^n P(y_i|x_i)$ above the threshold. We propose \ac{TSCP}, which leverages this intuition. 

\paragraph{TSCP.} Consider a model $f$ with the score of class $y$ and sample $x$ given by $f(x)[y]$. This can be seen as an approximate conditional distribution, and therefore, the product $\prod_{i=1}^n f(x_i)[y_i]$ can be used as the conformity score. The rest of the procedure is similar to \ac{SCP} \citet{Papadopoulos2002-hw}. First, on the calibration data, we compute the scores $s_i = 1 - \prod_{j=1}^n f(x^i_j[y^i_j])$ for $i=1,\dots, N_{\text{cal}}$. Next, we compute the finite sample corrected $\left(1 - \alpha\right)\left(1 + \frac{1}{N_{\text{cal}}}\right)$-quantile denoted by $q_\alpha$. For a new set of test samples $(x_1,\dots,x_n)$, include all the label pairs $(y_1,\dots,y_n)$ that satisfy $\prod_{j=1}^n f(x_j)[y_j] \geq 1-q_\alpha$. This method is detailed in Algorithm \ref{alg:split_cp}.  Note that the naive search over all possible label pairs will inflict the complexity of $O(\card{\gY}^n)$. However, we can leverage the compositional nature of the product score and implement it using dynamic programming. This approach prunes low probability branches  early on and reduces the complexity. We use this variant, presented in Appendix \ref{supp:algo_tscp}, Algorithm \ref{alg:tcp_split_dp}.

\begin{algorithm}[t]
\caption{Transductive Split Conformal Prediction (TSCP) with Score $f(x)[y]$}
\label{alg:split_cp}
\begin{algorithmic}[1]
\Require Calibration data $\mathcal{D}_{\text{cal}}=\{( (x^{i}_1, y^{i}_1),\dots,(x^{i}_n, y^{i}_n) )\}_{i=1}^{N_{\text{cal}}}$, miscoverage level $\alpha \in (0,1)$
\State Compute nonconformity scores on the calibration set:
$$s_i = 1 - \prod_{j=1}^n f(x^i_j)[y^i_j]), \quad \forall  (x^{i}_1, y^{i}_1),\dots,(x^{i}_n, y^{i}_n) \in \mathcal{D}_{\text{cal}}$$
\State Compute the empirical finite sample corrected $(1-\alpha)$ quantile:
$$
q_{\alpha} = \text{Quantile}_{\left(1 - \alpha\right)\left(1 + \frac{1}{N_{\text{cal}}}\right)} \left(\{s_i\}_{i=1}^{N_{\text{cal}}}\right)
$$
\Function{Predict}{$x_1,\dots,x_n$}
    \State Initialize prediction set $C(x) = \emptyset$; set $\tau\to 1-q_{\alpha}$
    \For{each label group $(y_1,\dots,y_n) \in \mathcal{Y}^{n}$}
        \If{$\prod_{j=1}^n f(x_j)[y_j] \geq \tau$}
            \State $C(x) \gets C(x) \cup \{(y_1,\dots,y_n)\}$
        \EndIf
    \EndFor
    \State \Return $C(x)$
\EndFunction
\end{algorithmic}
\end{algorithm}

\section{Related Works}
Conformal prediction \citep{vovk_algorithmic_2022} is a framework for confidence predictors with distribution-free coverage guarantees that rely only on the assumption that the samples are exchangeable. Some notable examples are split conformal prediction \citep{Papadopoulos2002-hw,lei2018distribution}, adaptive conformal prediction \citep{Romano2020-adaptive}, weighted conformal prediction \citep{tibshirani_conformal_2019,lei_conformal_2021} and localized conformal prediction \citep{Guan2023-localized-cp}---see \citet{angelopoulos_gentle_2021} for more details. Transdutive learning was introduced in  \citet{Gammerman1998-learning-transduc}, while transductive conformal prediction (TCP) was proposed in \citet{Vovk2013-transductive-pred} to generalize conformal prediction to multiple test examples. It is in this sense that we understand transductive learning. \citet{Vovk2013-transductive-pred} also discussed Bonferroni predictors as an information-efficient approach to transductive prediction. For a historical anecdote on the variations on the notion of transductive learning, going back to Vapnik, see 4.8.5 in \citet{vovk_algorithmic_2022}. Applications of transductive learning have been explored in ranking \citep{Fermanian2025-transduc-ranking}, which in itself includes many other use cases. Theoretical aspects of TCP were studied in \citet{Gazin2024-transduc-adaptive}, where the joint distribution of $p$-values for general exchangeable scores is derived. Although the applications of transductive learning are nascent, it provides a more general framework for studying confidence predictors.

Work on conformal prediction has focused on marginal and conditional guarantees, $p$- and $e$-value distributions, and extensions such as handling non-exchangeable samples \citep{Angelopoulos2024-theoretical,Foygel_Barber2021-limits, Gazin2024-transduc-adaptive,Vovk2012-ls, Bates2023-p-value, Marques-F2025-su, Vovk2024-seq-e-values,Vovk2023-tr,Grunwald2024-ra,Gauthier2025-qa}. In particular, two open research directions are relevant for this paper: first, the connection with hypothesis testing, and second, the theoretical bounds on the efficiency of conformal prediction.

In \citep{Correia2024-zq}, confidence prediction is framed as a list decoding problem in information theory. In that light, confidence predictors can be seen as list decoding for hypothesis testing with empirically observed statistics. The problem of Bayesian $M$-ary hypothesis testing with list decoding has been considered in \citet{asadi_kangarshahi_minimum_2023}, but assuming known probabilities and fixed list sizes.

Finally, on the efficiency of conformal prediction, \cite{Correia2024-zq} used the data processing inequality for $f$-divergences to get a lower bound on the logarithm of the expected prediction set size that mainly depends on the conditional entropy. In this work, we extend this study using a different class of bounds on hypothesis testing. Numerous information-theoretic bounds exist for hypothesis testing \citep{verdu_general_1994,Han2014-information-spectrum,polyanskiy_channel_2010,polyanskiy_arimoto_2010, Poor1995-mp, Chen2012-cf} with applications in finite-block-length analysis of Shannon capacity and source coding. Our bound in Theorem \ref{thm:verdu-han-transductive} generalizes \citet{verdu_general_1994} to variable-size list decoding; an extension to fixed-size list decoding is given in \citet{afser_statistical_2021}.

\begin{figure*}[ht!]
    \centering
    \includegraphics[width=\textwidth]{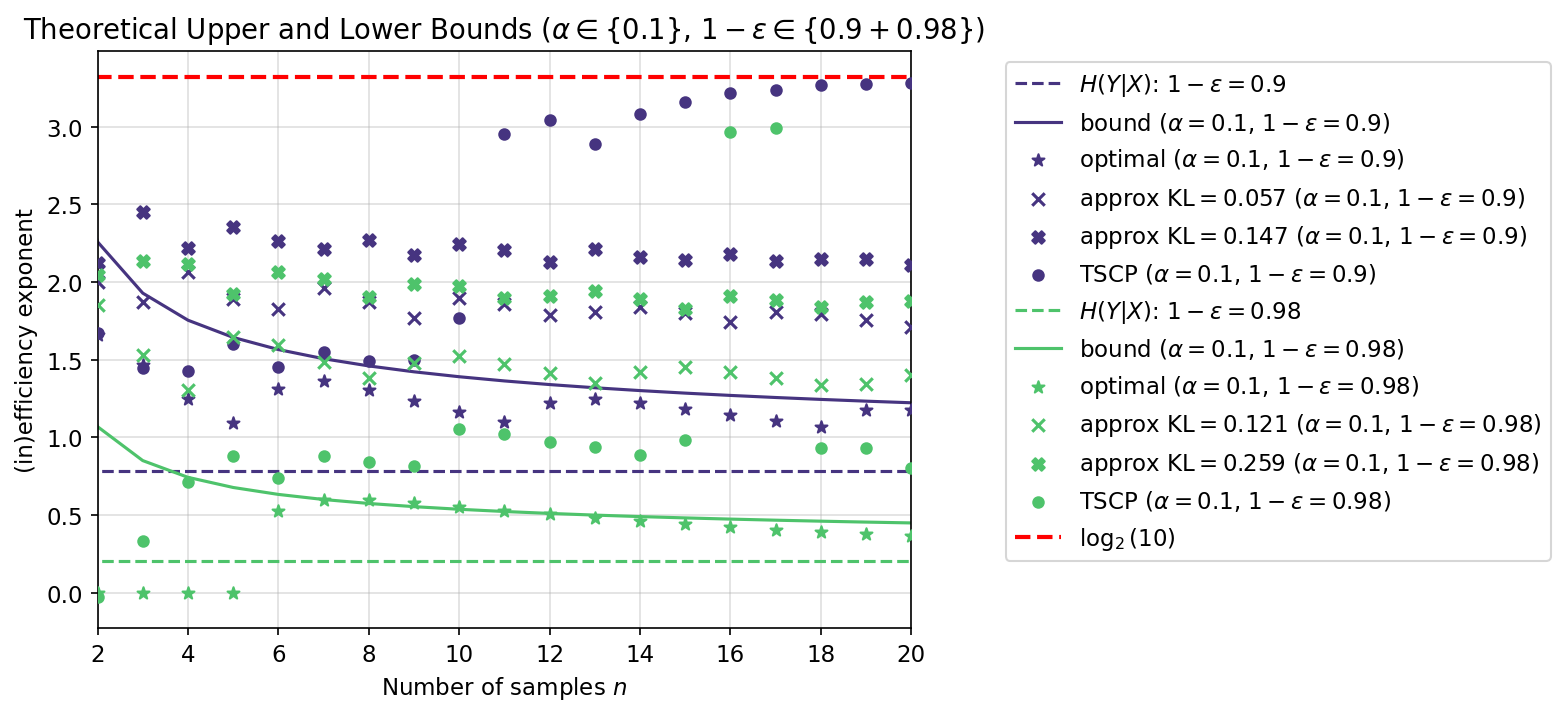}
    \caption{Comparison against theoretical bounds. The solid curve represents the predicted inefficiency bound, while the star markers indicate achievable points under the true conditional distribution. We additionally report the efficiency obtained when using approximate conditional distributions with controlled KL divergence from the ground truth. The performance of \ac{TSCP} is overlaid for direct comparison, demonstrating its alignment with the theoretical predictions.}
    \label{fig:theoretical_bound}
\vspace{-4mm}
\end{figure*}

\begin{figure*}[ht!]
    \centering
    \includegraphics[width=0.9\textwidth]{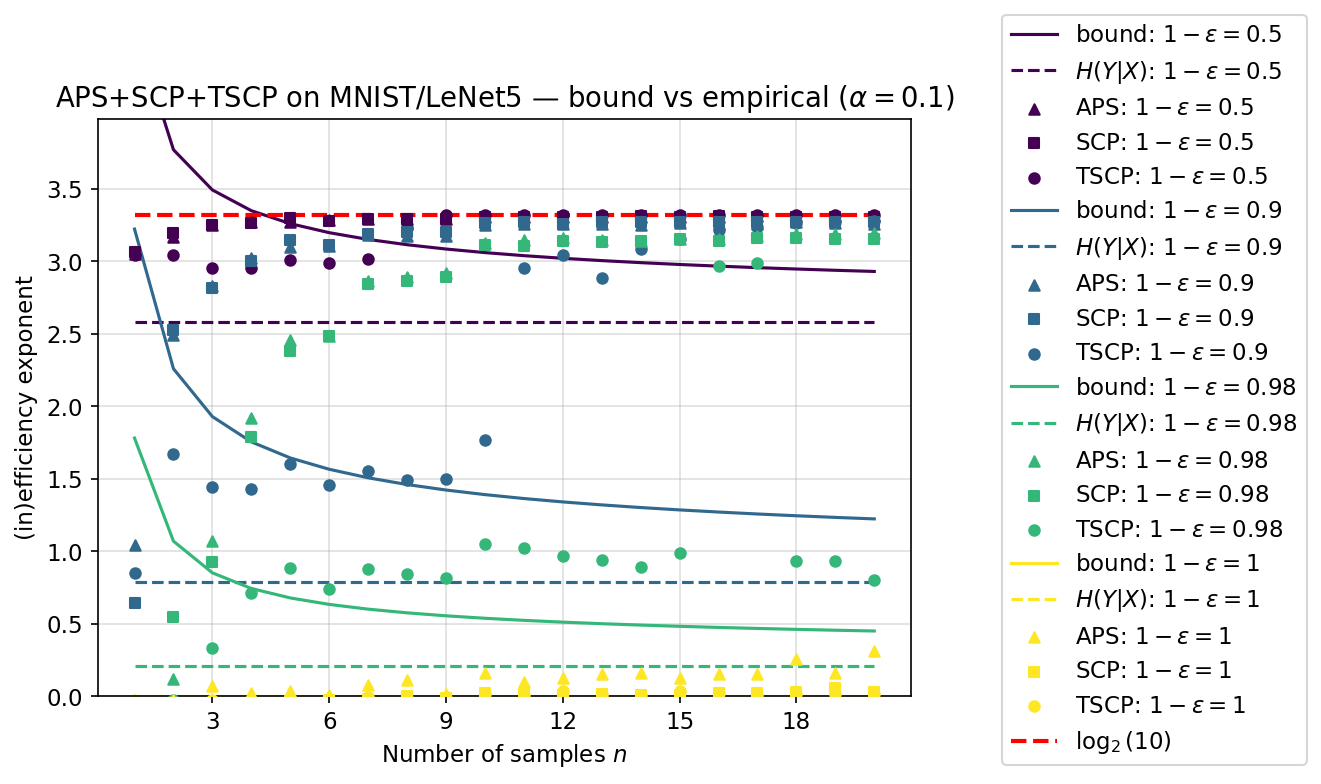}
    \caption{
    Comparison of the inefficiency exponent $\gamma_n$ for Bonferroni-corrected \ac{SCP}, \ac{APS}, and the proposed \ac{TSCP}. Our \ac{TSCP} closely tracks the theoretical bounds, whereas competing methods exhibit rapid degradation, manifested as exponential growth in inefficiency, with increasing label noise.
    }
    \label{fig:mnist_tscp_w_all}
\vspace{-3mm}
\end{figure*}
\section{Numerical Results}
In this section, we plot our theoretical bounds, and compare them with the bound obtained using Bonferroni-corrected conformal prediction methods, and the transductive SCP, which we introduced in the previous section. We have use various vision datasets (MNIST, FashionMNIST, CIFAR10, CIFAR100) with noisy labels to have control over the uncertainty: each label is kept with probability $1-\epsilon$ and changed to another class with probability $\epsilon/(N-1)$, where $N$ is the total number of classes. For this setup, we can easily compute $H(Y|X)$ and $\sigma$. For each dataset and noise level, we have trained a separate model: LeNet5 for MNIST and ResNet20 for the rest of the models. 

We use the efficiency rate $\gamma_{n,m}$ as the main measure. For the nonasymptotic bound, we use $H(Y|X)+\sigma Q^{-1}\paran{\alpha}/\sqrt{n}-\log_2 n/2n$, omitting constant terms $O(1)/n$ that vanish as $n$ grows. We use base 2 for the logarithms in the experiments. For the practical transductive experiments, we fix the number of samples in the calibration set to 180, which is resampled from the 0.8 portion of the validation set, while the rest is used for the test. For \ac{TSCP}, we devide the calibration set into groups of $n$ samples, each used to compute the conformity score. We have selected MNIST plots for the main paper, although the plots for other experiments carry a similar message. See Appendix \ref{supp:experiments} for more details.

\paragraph{Efficiency-Confidence trade-off.}  We plot the efficiency rate $\gamma_{n,m}$ as a function of the number of test sample $n$ (Theorem \ref{thm:nonasymptotic_transductive}) in Figure \ref{fig:theoretical_bound}. We have also plotted the achievability bound for known and approximate conditional distribution. We have perturbed the probability distribution with $\delta=0.015$ and $0.03$ and re-normalized it, which corresponds to the KL-divergence of $0.121$ and $0.259$. There is a persistent gap between conditional entropy (dashed line) and our bound (solid line), which closes only slowly. Even for hundreds of samples, our bound provides a better guideline for the efficiency rate (see Appendix \ref{supp:experiments} for more plots on the theoretical bounds). We have also included our \ac{TSCP} method in the plot, which for small $n$ aligns with the theoretical bound.

\paragraph{Bonferroni correction vs. Transductive SCP.} 
We compared our method with other transductive methods in Figure \ref{fig:mnist_tscp_w_all}. We used the Bonferroni-corrected predictor as explained in \citet{Vovk2013-transductive-pred}; see Appendix \ref{supp:tcp_vovk} for the  details. We have used Bonferroni-corrected \ac{SCP} \cite{Papadopoulos2002-hw} and \ac{APS} \cite{Romano2020-adaptive} for these experiments. First, note that Bonferroni prediction methods become inefficient quickly as $n$ increases. Particularly because the per-test confidence becomes more stringent. For example, for $\alpha=0.1$ and $n=20$, we need to have a confidence level of $0.005$ per sample. With a limited calibration set size, this quickly approaches the inefficient set prediction containing most labels. In contrast, our \ac{TSCP} method does not suffer from such a blow-up, and for $n<10$, it follows the theoretical bound. Since we are using a total of 180 samples in the calibration set, $n=10$ amounts to a transductive calibration set size of 18, which is quite small. This explains the similar, but still more graceful, blow up of \ac{TSCP} for larger $n$. Secondly, our approximation can be loose for smaller $n$ because of the ignored constant terms and relaxing the assumption $\alpha + \rho/\sqrt{n}\sigma^3 + \Delta\in[0,1]$.  
For larger $n$, the impact of these terms diminishes, and our approximate lower bound holds, providing a better lower bound than the conditional entropy. We have included a more extensive comparison of efficiency-coverage in Table \ref{tab:coverage_efficiency}, where \ac{TSCP} dominates the efficiency rate of almost all the existing Bonferroni-corrected schemes. Note that the effective calibration set size of \ac{TSCP} is $1/n$ of the Bonferroni corrected one, therefore, impacting the coverage. 

\begin{table*}[t]
\centering
\small
\caption{
Efficiency comparison of SCP, APS and TSCP for (MNIST: LeNet5), and (FashionMNIST, CIFAR10, CIFAR100: ResNet20) for different transductive sizes, $\alpha=0.1$, and $1-\epsilon=0.98$ as noise level. We compute the median and the expectation of the prediction-set size, and convert it to the efficiency rate using $\log_2(\cdot)/n$; coverage (Cov.) is the empirical joint coverage over the test batch.
}
\label{tab:coverage_efficiency}
\setlength{\tabcolsep}{4.5pt}
\begin{tabular}{lccc|ccc|ccc}
\toprule
\multirow{2}{*}{\textbf{Method}}
& \multicolumn{3}{c|}{$n=3$}
& \multicolumn{3}{c}{$n=6$} 
& \multicolumn{3}{c}{$n=9$} \\
\cmidrule(lr){2-4}\cmidrule(lr){5-7}\cmidrule(lr){8-10}
& $\gamma_{n,m}$ & Median & Cov.
& $\gamma_{n,m}$ & Median & Cov.
& $\gamma_{n,m}$ & Median & Cov. \\
\midrule
SCP (MNIST)
& 0.9261 & 0.0000 & 0.9400
& 2.4865 & 1.2974 & 0.9400
& 2.8948 & 2.3454 & 0.9000 \\

APS (MNIST)
& 1.0683 & 0.0000 & 0.8400
& 2.4883 & 1.5429 & 0.9000
& 2.9187 & 2.5481 & 0.9200 \\

\midrule
\textbf{TSCP (MNIST)}
& \textbf{0.3333} &  {0.0000} & {0.8400}
& \textbf{0.7410} &  {0.6950} & {0.8800}
& \textbf{0.8133} &  {0.6379} & {0.9600} \\
\bottomrule
\midrule
SCP (FMNIST)
& 1.4042 & 0.3333 & 0.8800
& 2.7644 &  1.6912 & 0.8200
& 3.1409 & 2.4008 & 0.9600 \\

APS (FMNIST)
& 1.3742 & 0.7740 & 0.9200
& 2.6491 & 1.9593 & 0.9000
&  3.0022 & 2.5863 & 0.8800 \\

\midrule
\textbf{TSCP (FMNIST)}
& \textbf{0.7873} &  {0.3333} & {0.9000}
& \textbf{1.0648} &  {0.8809} & {0.9000}
& \textbf{1.1134} &  {0.8809} & {0.9200} \\
\bottomrule
\midrule
SCP (CIFAR10)
& 1.6788 & 0.8617 & 0.9000
& 2.7067 & 1.5460 & 0.9200
& 2.9777 & 2.5339 & 0.9200 \\

APS (CIFAR10)
& 1.5871 & 0.9358 & 0.9000
& 2.5133 & 1.7559 & 0.8600
& 2.9327 & 2.5280 & 0.9000 \\

\midrule
\textbf{TSCP (CIFAR10)}
& \textbf{0.9952} &  {0.8617} & {0.8600}
& \textbf{1.3562} &  {1.1167} & {0.9200}
& \textbf{1.2414} & { 1.1942} & {0.9400} \\
\bottomrule
\midrule
SCP (CIFAR100)
& 4.8998 & 3.6475 & 0.9600
& 6.2604 & 4.8781 & 0.9200
& 6.5391 &  {6.6439} & 0.9200 \\

APS (CIFAR100)
& 4.9237 & 3.6082 & 0.9000
& 6.0093 & 4.8461 & 0.9400
& \textbf{6.2943} &  {5.3957} & 0.9800 \\

\midrule
\textbf{TSCP (CIFAR100)}
& \textbf{3.4176} & {2.8171} & {0.8800}
& \textbf{5.8699} & { 2.8573} & {0.9200}
&  {6.6081} & {6.6439} & {0.8800} \\
\bottomrule
\bottomrule
\end{tabular}

\vspace{0.5em}
\parbox{0.92\linewidth}{\footnotesize
\textbf{Notes.} The upper bound on the efficiency rate is $\log_2(\text{number-of-classes})$, which is 3.3219 for 10 classes, and 6.6438 for 100 classes.
}
\vspace{-5mm}
\end{table*}

\section{Conclusion}
We established new theoretical bounds that characterize the trade-off between efficiency and confidence in transductive conformal prediction, revealing fundamental limitations that hold across transductive confidence predictors. Our results show that any non-trivial confidence level necessarily leads to exponential growth in prediction set size, with rates governed by conditional entropy and dispersion. Practitioners can use this framework to evaluate the efficiency of transductive methods in settings where these information-theoretic quantities are computable. 

Our analysis also exposes the inefficiency of Bonferroni-based approaches and motivates more principled transductive predictors. Along this line, we introduced transductive split conformal prediction (TSCP), which empirically achieves efficiency rates aligned with the theoretical characterization. Future work includes extending these bounds beyond the i.i.d.\ setting and further investigating statistical and computational trade-offs in transductive inference.

\bibliography{bibliography}
\bibliographystyle{plainnat}

\appendix

\section{Proofs of Section \ref{sec:eff_cof_trade_off}}
\label{sec:proof_off_cof_trade_off_section}
 In Shannon Theory, it is known that Fano's inequality cannot be used to prove strong converse results and establish the phase transition at the Shannon capacity.
Other inequalities in information theory involve the underlying probability distribution or information density terms \citep{verdu_general_1994,Han2014-information-spectrum}. For example, Theorem 4 in \citep{verdu_general_1994} states that:
    \[
\P(P(X|Y)\leq \beta ) \leq \epsilon + \beta,
    \]
where $\epsilon$ is the error probability of a code for a channel $P_{Y|X}$ and $\beta$ is an arbitrary number in $[0,1]$. These bounds can sometimes lead to tighter results compared to Fano's inequality, as indicated in \citep{polyanskiy_channel_2010}.  We start from this observations and derive the new inequalities for transductive conformal prediction. Note that these inequalities are new and do not follow from existing works.
\subsection{Proof of Theorem \ref{thm:verdu-han-transductive}}
\label{proof:thm:verdu-han-transductive}
\begin{proof}
    To simplify the notation for the proof, we denote $\mX:= X_{m+1}^{m+n}$, $\mY:=Y_{m+1}^{m+n}$ and $\mZ:=Z_{1}^{m}$. we assume $(\mX,\mY)$ are drawn from the distribution $P$.     For each $\mX\in \gX^n$, define the set:
    \[
    B_{\mX} = \{ \mY:P(\mY|\mX) \leq \beta \}.
    \]
    $P(\mY|\mX)$ is the conditional distribution induced by $P$.
    The proof follows the steps below:
    \begin{align*}
        \P( P(\mY|\mX) \leq \beta) & = \int P(B_{\mX}|\mX)P(\mX)d\mX  = \int P(\mX,B_{\mX})d\mX  = \int P(\mX,\mZ,B_{\mX})d\mX d\mZ\\
        & =  \int \int P(\mX,\mZ,B_{\mX} \cap \Gamma^{\alpha}(\mZ,\mX)^c)d\mX d\mZ  + \int \int P(\mX,\mZ,B_{\mX}\cap \Gamma^{\alpha}(\mZ,\mX))d\mX d\mZ\\
        &\leq  \int \int P(\mX,\mZ, \Gamma^{\alpha}(\mZ,\mX)^c)d\mX d\mZ  + \int \int P(\mX,\mZ,B_{\mX}\cap \Gamma^{\alpha}(\mZ,\mX))d\mX d\mZ \\
        &  \stackrel{(1)}{\leq}   \alpha +  \int \int P(\mX,\mZ,B_{\mX}\cap \Gamma^{\alpha}(\mZ,\mX))d\mX d\mZ \\
        & =  \alpha +  \int \int P(\mX,\mZ)P(B_{\mX}\cap \Gamma^{\alpha}(\mZ,\mX)|\mX,\mZ)d\mX d\mZ \\
        & =  \alpha + 
        \int \int P(\mX,\mZ)\paran{\sum_{\mY, \mY\in B_{\mX}\cap \Gamma^{\alpha}(\mZ,\mX)} P(\mY|\mX,\mZ)}d\mX d\mZ
        \\
        &  \stackrel{(2)}{\leq}  \alpha +  \int\int P(\mX,\mZ)\paran{\sum_{\mY, \mY\in B_{\mX}\cap \Gamma^{\alpha}(\mZ,\mX)}  \beta} d\mX d\mZ\\
        & \leq  \alpha +  \int P(\mX,\mZ) \beta |\Gamma^{\alpha}(\mZ,\mX)| d\mX d\mZ  = \alpha + \beta\E(|\Gamma^{\alpha}(\mZ,\mX)|).
    \end{align*}
The inequality (1) follows from the confidence assumption:
\[
P\paran{Y_{m+1}^{m+n}\notin \Gamma^{\alpha}(Z_{1}^{m},X_{m+1}^{m+n})} = \int \int P(\mX,\mZ, \Gamma^{\alpha}(\mZ,\mX)^c)d\mX d\mZ \leq \alpha.
\]
The inequality (2) follows from the independence of $\mZ$ and $(\mX,\mY)$ and the definition of $B_{\mX}$.
\end{proof}

\begin{remark} Note that Theorem \ref{thm:verdu-han-transductive} can be written as:
    \begin{equation}
         \sup_{\beta\in[0,1]}\frac{\P( P(Y_{m+1}^{m+n}|X_{m+1}^{m+n}) \leq \beta )  - \alpha}{\beta}\leq \E(|\Gamma^{\alpha}(Z_{1}^{m},X_{m+1}^{m+n})|) 
    \end{equation}
This bound means that the prediction set size is large even if for a single $\beta$, the left hand side is large.  As an example, suppose that the conditional probability distribution is not concentrated around a point and has large spread. In other words, the uncertainty is high, and we can find many labels with equally high but numerically small probabilities. In this case, we can choose a small $\beta$ that yields a very high probability  of $ P(Y_{m+1}^{m+n}|X_{m+1}^{m+n}) \leq \beta$. Therefore, the prediction set size scales mainly with with $1/\beta$ and is expected to be large as $\beta$ is small. Intuitively, the left hand side of Theorem \ref{thm:verdu-han-transductive} measures the intrinsic uncertainty and right hand side measures the prediction set size.
\end{remark}
\subsection{Proof of Theorem \ref{thm:fundamental_asymp_transductive}}

\begin{proof}
We use Theorem \ref{thm:verdu-han-transductive} to prove the results. The choice of $\beta$ can be important. Let's choose $\beta=e^{-n(H(Y|X)-\delta)}$ where $H(Y|X)$ is the conditional entropy, and $\delta$ is any non-negative number. With standard manipulations, we obtain the following result from Theorem \ref{thm:verdu-han-transductive}:
\begin{align}
     \P\paran{\frac{1}{n}\log {P(Y_{m+1}^{m+n}|X_{m+1}^{m+n})} \leq -H(Y|X)+\delta  }  \leq \alpha_n + e^{-n(H(Y|X)-\delta)}\E(|\Gamma^{\alpha}(Z_{1}^{m},X_{m+1}^{m+n})|).
\label{ineq:cond_entropy_density}
\end{align}
 Since the test samples are i.i.d., the term $\log {P(Y_{m+1}^{m+n}|X_{m+1}^{m+n})}$ can be decomposed as:
\[
\frac{1}{n}\log {P(Y_{m+1}^{m+n}|X_{m+1}^{m+n})} =\frac{1}{n} \sum_{i=1}^n
\log {P(Y_{m+i}|X_{m+i})}.
\]
From law of large numbers, as $n$ goes to infinity, the sum converges almost surely to the negative conditional entropy between the input $X$ and the label $Y$, namely $-H(Y|X)$. This means that the probability on the left hand side of \eqref{ineq:cond_entropy_density} goes to one. We have:
\[
1 \leq \liminf_{n\to\infty } (\alpha_n + e^{-n(H(Y|X)-\delta)}\mathbb{E}(|\Gamma^{\alpha}(Z_{1}^{m},X_{m+1}^{m+n})|)).
\]
We have two cases:
\begin{itemize}
    \item Case 1: if $\gamma^{+}_m < H(Y|X)$, then we have:
    \[
    \liminf_{n\to\infty}e^{-n(H(Y|X)-\delta)}\E(|\Gamma^{\alpha}(Z_{1}^{m},X_{m+1}^{m+n})|) = 0.
    \]
    This means that $\liminf_{n\to\infty}\alpha_n= 1$, which means $\lim_{n\to\infty}\alpha_n= 1$, and the confidence goes to zero.
\item Case 2: for non-trivial asymptotic confidence, $\limsup_{n\to\infty}\alpha_n$ is strictly below one. For the inequality to hold, the second term needs to be non-vanishing, that is $\gamma^{-}_m > H(Y|X)-\delta$, namely $\gamma^{-}_m \geq H(Y|X)$.
\end{itemize}
The proof follows accordingly.
\end{proof}
\subsection{Proof Of Theorem \ref{thm:nonasymptotic_transductive}}
\label{proof:thm:nonasymptotic_transductive}
\begin{proof}
We use Berry-Esseen central limit theorem for the proof. 
\begin{theorem}[Berry-Esseen] Let $X_i$, $i\in[n]$ be i.i.d. random
    variables with $\E(X_i)=\mu, \text{Var}(X_i)=\sigma^2, \rho=\E(|X_i-\mu|^3)$. Then we have for any $t\in\R$:
    \[
    \card{
    \P\paran{
    \frac{1}{\sqrt{n}}\sum_{i=1}^n\frac{(X_i-\mu)}{\sigma}\geq t
    }
    - Q(t)
    }\leq \frac{\rho}{\sqrt{n}\sigma^3}.
    \]
\label{thm:berry-esseen}
\end{theorem}
We use Theorem \ref{thm:verdu-han-transductive} as starting point:
\begin{align*}
\P\paran{ {P(Y_{m+1}^{m+n}|X_{m+1}^{m+n})} < \beta  }   & =
\P\paran{ \frac{1}{\sqrt{n}}\log{P(Y_{m+1}^{m+n}|X_{m+1}^{m+n})}  \leq \frac{1}{\sqrt{n}}\log\beta  }\\
& =
\P\paran{ \frac{1}{\sqrt{n}}\sum_{i=1}^n \log{P(Y_{m+i}|X_{m+i})}  \leq \frac{1}{\sqrt{n}}\log\beta  }\\
& =
\P\paran{ \frac{1}{\sqrt{n}}\sum_{i=1}^n \frac{(\log{P(Y_{m+i}|X_{m+i})}-\mu)}{\sigma}  \leq \frac{1}{\sigma\sqrt{n}}\log\beta -\frac{\sqrt{n}\mu}{\sigma} }\\
&\geq Q\paran{
\frac{-1}{\sigma\sqrt{n}}\log\beta +\frac{\sqrt{n}\mu}{\sigma}
}-\frac{\rho}{\sqrt{n}\sigma^3}
\end{align*}
where $\mu = -H(Y|X)$ and the last step follows from Berry-Esseen.
Now choose $ \beta = \exp\paran{
{n}\mu-Q^{-1}(\epsilon)\sigma \sqrt{n}
}$, which implies that:
\[
 Q\paran{
\frac{-1}{\sigma\sqrt{n}}\log\beta +\frac{\sqrt{n}\mu}{\sigma}
}=\epsilon
\]

We get the following simplified inequality for any $\epsilon$:
\begin{align*}
  \epsilon\leq \alpha + \frac{\rho}{\sqrt{n}\sigma^3} +
  \exp\paran{
{n}\mu-Q^{-1}(\epsilon)\sigma \sqrt{n} + n\gamma_{n,m}}
\end{align*}
choose $\epsilon=\alpha + \frac{\rho}{\sqrt{n}\sigma^3} + \Delta$ for any $\Delta>0$, such that $\epsilon\in(0,1)$. Then, we get the result:
\begin{align*}
  \log\Delta - {n}\mu+Q^{-1}(\alpha + \frac{\rho}{\sqrt{n}\sigma^3} + \Delta)\sigma \sqrt{n} \leq { n\gamma_{n,m}}.
\end{align*}
\end{proof}

\paragraph{Deriving An Approximate Bound.} Note that $Q(x)$ is non-increasing and $(1/\sqrt{2\pi})-$Lipschitz (given that $Q'(x)$ is negative Gaussian density function). Therefore, we have:
\[
 \Q^{-1}\paran{\alpha + \frac{\rho}{\sqrt{n}\sigma^3} + \Delta}\geq
  \Q^{-1}\paran{\alpha} - \frac{1}{\sqrt{2\pi}}\left(\frac{\rho}{\sqrt{n}\sigma^3}+\Delta\right),
\]
We can choose $\Delta=\Delta'/\sqrt{n}$, and show that there is a constant $C_0$ such that:
\[
 \Q^{-1}\paran{\alpha + \frac{\rho}{\sqrt{n}\sigma^3} + \Delta}\geq
  \Q^{-1}\paran{\alpha} - \frac{C_0}{\sqrt{n}}.
\]
which gives the approximate exponent:
\begin{align*}
  \log\Delta' - C_0\sigma -\frac{1}{2}\log n - {n}\mu+Q^{-1}(\alpha)\sigma \sqrt{n} \leq { n\gamma_{n,m}}.
\end{align*}
The term $ \log\Delta' - C_0\sigma$ is constant and hence $O(1)$. The final result follows as:
\[
 {  n\gamma_{n,m}\geq nH(Y|X)+
 \sqrt{n}\sigma
  \Q^{-1}\paran{\alpha}  -\frac{\log n }{2} + O(1).}
\]

\begin{remark}
Looking at the proof more closely, a precise statement of the bound is as follows:
\begin{align*}
  \epsilon\leq \alpha + \frac{\rho}{\sqrt{n}\sigma^3} +
  \exp\paran{
{n}\mu-Q^{-1}(\epsilon)\sigma \sqrt{n} + n\gamma_{n,m}}.
\end{align*}
The choice of $\epsilon=\alpha + \frac{\rho}{\sqrt{n}\sigma^3} + \Delta$ for any $\Delta>0$ needs to satisfy $\epsilon\in(0,1)$. Our approximation ignores this condition, which can lead to vacuous results. In certain cases, even $\alpha + \frac{\rho}{\sqrt{n}\sigma^3}$  can be outside $(0,1)$, which yields a vacuous bound, although asymptotically as $n\to\infty$, the term will always be within the desired range. Another component is $\Delta$, which is  between $(0,1)$. This means that $\log\Delta <0$, and therefore, the actual bound is smaller that $nH(Y|X)+
 \sqrt{n}\sigma
  \Q^{-1}\paran{\alpha}$. Again, as $n$ increases, these impacts vanish, and the bound should be non-vacuous.
\label{remark:tightness_bound}
\end{remark}
\section{Efficiency-Confidence Trade-off for General Notions of Efficiency}
\label{sec:eff-conf-general}
In \cite{vovk_algorithmic_2022}, Section 3.1, various criteria for efficiency has been discussed such as sum, number, unconfidence, fuzziness, multiple, and excess criterion. Our notion of efficiency based on the prediction set size is the number criterion in the transductive setting. Here. we can generalize the result for a general criterion of efficiency that can be expressed by a measure (not necessarily a probability measure). We start with the following more general result. 
\begin{theorem}
    Consider a transductive conformal predictor $\Gamma^{\alpha}(Z_{1}^{m},X_{m+1}^{m+n})$ given a labeled dataset  $Z_{1}^{m}$ and test samples $X_{m+1}^{m+n}$ with unknown labels $Y_{m+1}^{m+n}$. If the predictor has the confidence $1-\alpha$, then for any positive $\beta$ and any \textit{measure} $Q$, we have:
    \begin{align*}
        \P(P(Y_{m+1}^{m+n}|X_{m+1}^{m+n}) & \leq \beta Q(Y_{m+1}^{m+n}|X_{m+1}^{m+n}) ) \leq \\
        & \alpha +  \beta \int P(X_{m+1}^{m+n},Z_{1}^{m})  Q(\Gamma^{\alpha}(Z_{1}^{m},X_{m+1}^{m+n})|X_{m+1}^{m+n})dX_{m+1}^{m+n} dZ_{1}^{m} . 
    \end{align*}   
\label{thm:verdu-han-withQ}
\end{theorem}
\begin{proof}
    We follow the idea of the proof given in \ref{proof:thm:verdu-han-transductive}. Define:
    \[
    B_{\mX} := \{ \mY:P(\mY|\mX) \leq \beta Q(\mY|\mX) \}.
    \]
    We need to modify the last step the of the proof as follows:
    \begin{align*}
        \P( P(\mY|\mX) < \beta Q(\mY|\mX) ) & = \int P(\mX,B_{\mX})d\mX  \\
        &  \leq  \alpha +  \int\int P(\mX,\mZ)\paran{\sum_{\mY, \mY\in B_{\mX}\cap \Gamma^{\alpha}(\mZ,\mX)}  \beta Q(\mY|\mX) } d\mX d\mZ\\
        & \leq  \alpha +  \beta \int P(\mX,\mZ)  Q(\Gamma^{\alpha}(\mZ,\mX)|\mX)d\mX d\mZ .
    \end{align*}
\end{proof}

Finally, we can extend the result to the non-asymptotic case. The measure $Q$ can be the one represented by the model or can be any notion of efficiency as before. A similar trick has been used in Eq. (102) of \citep{polyanskiy_channel_2010} in their meta-converse analysis.

\begin{theorem}
    For a transduction conformal predictor with confidence $1-\alpha$. Define:
    \begin{align*}
   \gamma^{(Q)}_{n,m}:= &\frac{1}{n}\log \int P(X_{m+1}^{m+n},Z_{1}^{m})  Q(\Gamma^{\alpha}(Z_{1}^{m},X_{m+1}^{m+n})|X_{m+1}^{m+n})dX_{m+1}^{m+n} dZ_{1}^{m}\\
   =&\frac{1}{n}\log\E_{X_{m+1}^{m+n},Z_{1}^{m}}\paran{
   Q(\Gamma^{\alpha}(Z_{1}^{m},X_{m+1}^{m+n})|X_{m+1}^{m+n})
   },
    \end{align*}
    for any measure $Q(Y|X)$ satisfying $Q(Y_{m+1}^{m+n}|X_{m+1}^{m+n})=\prod_{i=1}^n Q(Y_{m+i}|X_{m+i})$.
    Then for any $n$ and $\Delta>0$ such that 
$\alpha + \frac{\rho}{\sqrt{n}\sigma^3} + \Delta\in[0,1]$, we have:
    \begin{align*}
 \log\Delta + n\mu+
 \sqrt{n}\sigma
  \Q^{-1}\paran{\alpha + \frac{\rho}{\sqrt{n}\sigma^3} + \Delta}  \leq { n\gamma^{(Q)}_{n,m}}
\end{align*}
where $Q(\cdot)$ is the Q-function, and:
\begin{align}
    \mu&:=\E\paran{\log
    \frac{P(Y|X)}{Q(Y|X)}
    }\\
    \sigma&:=\mathrm{Var}\paran{\log
    \frac{P(Y|X)}{Q(Y|X)}}^{1/2}=\paran{\E\paran{\log
    \frac{P(Y|X)}{Q(Y|X)} - \mu}^2}^{1/2}\\
    \rho &:=\E\paran{\card{\log
    \frac{P(Y|X)}{Q(Y|X)} - \mu}^3}.
\end{align}
\end{theorem}
The proof follows the exact same steps as in the proof given \ref{proof:thm:nonasymptotic_transductive}, and we omit it. 

Note that if $Q(\cdot)$ is a probability measure, the term $\mu$ is given by $\E\paran{
    D_{KL}\paran{P(Y|X)\Vert Q(Y|X)}|X
    }$. One insight from the above theorem is that the exponent of the transductive prediction efficiency, measured using a probability measure, is asymptotically the KL-divergence between the used measure and ground truth conditional probability.

\paragraph{Other efficiency metrics.} Various other efficiency metrics are used in the literature. Most of these are discussed in Chapter 3 of \cite{vovk_algorithmic_2022}. Two categories are noteworthy.  One is called observed criteria where the efficiency is measured based on the observation of the data label. The other category of criteria measures the efficiency using the prediction set only. Our derivation here can be applied to observed criterion of efficiency assuming the ratios are well defined. We consider some of these measures here. 

Consider first N-criterion (N for number), defined as:
\[
Q(\Gamma^{\alpha}(X_{m+1}^{m+n})) = \frac{1}{n}\sum_{i=1}^n \card{\Gamma_{i}^{\alpha}(X_{m+1}^{m+n})},
\]
where $\Gamma_{i}^{\alpha}(X_{m+1}^{m+n})$ is defined as the set of different labels predicted in $\Gamma_{i}^{\alpha}(X_{m+1}^{m+n})$ for the sample $m+i$. For this choice of $Q$, we get:
\[
Q(Y_{m+1}^{m+n}|X_{m+1}^{m+n}) = \frac{1}{n} \sum_{i=1}^n 1 = 1.
\]
This means that for this choice, it is still the best to simply threshold the conditional probability. The fundamental bounds on the efficiency rate, therefore, remains very similar.

Next, consider S-criterion (S for sum). This is defined as sum of $p$-values for all the labels across test samples. We consider a modified version defined as:
\[
Q(\Gamma^{\alpha}(X_{m+1}^{m+n})) = \sum_{y_{m+1}^{m+n}\in \Gamma^{\alpha}(X_{m+1}^{m+n})} p_{y_{m+1}^{m+n}},
\]
where $p_{y_{m+1}^{m+n}}$ is computed given the calibration set $Z_1^m$ and $X_{m+1}^{m+n}$. In this case, we have:
\[
Q(Y_{m+1}^{m+n}|X_{m+1}^{m+n}) = p(Y_{m+1}^{m+n}|X_{m+1}^{m+n}),
\]
where we made the conditioning of $p$-value on the test data explicit in the notation. In this case, the fundamental limits will be determined by the following ration:
\[
\log\frac{P(Y|X)}{p(Y|X)},
\]
where the small $p$ represents the $p$-value of $Y$ given $X$. $p$-values are between 0 and 1 but not a probability measure. 

We introduce a third efficiency criterion called R-criterion (R for risk). This notion measures the risk of the labels in the prediction set. Let's consider the autonomous driving use case and the object detection application. Different objects can lead to different course of actions, each incurring different costs. Therefore, we might want to measure the average risk incurred by the prediction set using a risk measure. This application corresponds to the weighted set size as the efficiency measure where the weight of each label is proportional to its risk. We define this risk as:
\[
R(\Gamma^{\alpha}(X_{m+1}^{m+n})) = \sum_{y \in \Gamma^{\alpha}(X_{m+1}^{m+n})} R(y).
\]
We leave the risk function $R(\cdot)$ quite general so it can apply to a single label or a sequence of labels. We can replace $Q$ with $R$ in the above result. Note that $R(Y|X) = R(Y)$. The fundamental limits, in this case, are the moments of the following ratio:
\[
\log\frac{P(Y|X)}{R(Y)}.
\]
In other words, the conditional probability needs to be scaled with the risk function for optimal performance.

\section{Discussion on Achievability on Non-asymptotic Bounds on Efficiency}
\label{proof:supp:achievability}
\begin{proof}
We start with the following lemma, which gives a bound on the expected set size.
\begin{lemma}
Consider two spaces for $\gX$ and $\gY$ with a joint probability distribution $P(x,y)$ defined over the product space for $x\in\gX$ and $y\in\gY$. Define the set $A_x$ for $x\in\gX$ as follows:
\[
A_x = \{y: P(y|x) \geq \beta\}.
\]
Then:
\[
\E_{X}[|A_X|]\leq \frac{1}{\beta}.
\]
\label{lemma:set_size}
\end{lemma}
\begin{proof}
The proof is as follows:
\begin{align}
\P(A_X) & = \sum_{(x,y)\in\gX\times\gY} P(x,y) \mathbf{1}(y\in A_x)\\
 &\geq\sum_{(x,y)\in\gX\times\gY} P(x) \beta \mathbf{1}(Y\in A_x)\\
 & = \sum_{x\in\gX} P(x) \sum_{y\in A_x}\beta  \mathbf{1}(Y\in A_x) \\
 & = \beta\E[|A_X|],
\end{align}
where we used to inequality $\P(y|x)\geq \beta$ for $y\in A_x$. Using $\P(A_X)\leq 1$, we get the inequality.
\end{proof}

Now, we just need to pick $\beta$ such that the probability of $A_X$ satisfies the required confidence level. To do so, consider the set of labels:
\[
\Gamma^{\alpha}(x_{1}^{n}):=\{y_{1}^{n}: {P(y_{1}^{n}|x_{1}^{n})} \geq \beta \}.
\]
When $(X_i,Y_i)$ are independently and identically drawn from $P(X,Y)$, we can use the Berry-Esseen central limit theorem, Theorem \ref{thm:berry-esseen}, to bound the probability of the set $\Gamma^{\alpha}(x_{1}^{n})$. The probability of error is the probability that the labels $Y_{m+1}^{m+n}$ do not belong to the set  $\Gamma^{\alpha}(X_{1}^{n})$. It can be bounded as follows.
\begin{align*}
\P\paran{ {P(Y_{1}^{n}|X_{1}^{n})} \leq \beta  }   & =
\P\paran{ \frac{1}{\sqrt{n}}\sum_{i=1}^n \frac{(\log{P(Y_{i}|X_{i})}-\mu)}{\sigma}  \leq \frac{1}{\sigma\sqrt{n}}\log\beta -\frac{\sqrt{n}\mu}{\sigma} }\\
&\leq Q\paran{
\frac{-1}{\sigma\sqrt{n}}\log\beta +\frac{\sqrt{n}\mu}{\sigma}
}+ \frac{\rho}{\sqrt{n}\sigma^3}.
\end{align*}
As before $\mu = -H(Y|X)$, $\sigma$ is the variance of the log probability,  and the last step follows from Berry-Esseen theorem. To guarantee the confidence level $\alpha$, we need to choose $\beta$ as follows:

\[
 Q\paran{
\frac{-1}{\sigma\sqrt{n}}\log\beta +\frac{\sqrt{n}\mu}{\sigma}
}+ \frac{\rho}{\sqrt{n}\sigma^3} = \alpha
\]
which yields the following choice of $\beta$
\[
\beta =\exp\paran{n\mu - Q^{-1}(\alpha - \frac{\rho}{\sqrt{n}\sigma^3})\sigma \sqrt{n}}
\]
This is conditioned on $\alpha - \frac{\rho}{\sqrt{n}\sigma^3}\in(0,1)$, which might not hold for smaller $n$. Indeed, we need to have:
\[
n > \left(\frac{\rho}{\alpha\sigma^3}\right)^2.
\]
The function $Q^{-1}(\cdot)$ is $1/\sqrt{2\pi}$-Lipschitz, and we have:
\[
Q^{-1}(\alpha - \frac{\rho}{\sqrt{n}\sigma^3}) \leq Q^{-1}(\alpha) + \frac{1}{2\pi}\frac{\rho}{\sqrt{n}\sigma^3}.
\]
Therefore:
\[
\beta \leq \exp\paran{n\mu - 
\sigma \sqrt{n} Q^{-1}(\alpha) + \frac{1}{2\pi}\frac{\rho}{\sigma^2} 
}
\]
With this choice of $\beta$, the expected set size is bounded using the above lemma as:
\[
\E[|\Gamma^{\alpha}(X_{1}^{n})|] \leq \exp\paran{-n\mu + 
\sigma \sqrt{n} Q^{-1}(\alpha) - \frac{1}{2\pi}\frac{\rho}{\sigma^2},
}
\]
which yields the result. 
\end{proof}

Our derivation does not contain the logarithmic terms, $-\frac{1}{2}\log n$ that appears in the lower bound. We can use a different technique, similar to the one used in \citep{Kontoyiannis2014-px} for the lossless compression case, to get this term as well.

\subsection{Achievability for approximate distributions}
\label{app:approx_dist}
Suppose that we have access to the approximation of the conditional distribution $P(y|X)$, given by $Q(y|x)$, and the transductive confidence sets are constructed using $Q$ as follows:
\[
\Gamma_Q^{\alpha}(x_{1}^{n}):=\{y_{1}^{n}: {Q(y_{1}^{n}|x_{1}^{n})} \geq \beta \}.
\]
We can provide a bound on the confidence sets constructed in this way. The bound contains different divergences between $P(Y|X)$ and $Q(Y|X)$. 

\begin{theorem}
Consider the confidence set $\Gamma_Q^{\alpha}(X_1^n)$ defined above. Then, there is a choice of $\beta$  that achieves the confidence $1-\alpha$ at the efficiency rate $\gamma_n$ satisfying:
\[
n\gamma_n \leq \log\paran{1+n\E[\text{TV}\paran{P(\cdot|X),Q(\cdot|X)}]} + n H(Y|X) + n \E[D(P(\cdot|X)\Vert Q(\cdot|X))] +
\sqrt{n}\sigma  \mathrm{Q}^{-1}(\alpha) + O(1),
\]
assuming $\alpha\geq \rho/\sqrt{n}\sigma^3$, and:
\begin{align*}
    \sigma & =\paran{\mathrm{Var}\paran{\log Q(Y|X)}}^{1/2}\\
    \rho &=\E\paran{\card{\log Q(Y|X) -\mu}^3}.
\end{align*}
All expectations are w.r.t. data distribution $P$.
\end{theorem}
\begin{proof}
First, we can use Lemma \ref{lemma:set_size} to see that:
\[
\card{\Gamma_Q^{\alpha}(x_{1}^{n})} \leq \frac{Q^n (\Gamma_Q^{\alpha}(x_{1}^{n}))}{\beta}.
\]
Then:
\[
Q (\Gamma_Q^{\alpha}(x_{1}^{n})) \leq P (\Gamma_Q^{\alpha}(x_{1}^{n})) + \card{Q (\Gamma_Q^{\alpha}(x_{1}^{n})) - P (\Gamma_Q^{\alpha}(x_{1}^{n}))} \leq P (\Gamma_Q^{\alpha}(x_{1}^{n})) + \text{TV}\paran{P(\cdot|x_{1}^{n}),Q(\cdot|x_{1}^{n}},
\]
which implies that: 
\[
\E[\card{\Gamma_Q^{\alpha}(X_{1}^{n})}]\leq \frac{P (\Gamma_Q^{\alpha}(X_{1}^{n})) + \E[\text{TV}\paran{P(\cdot|X_{1}^{n}),Q(\cdot|X_{1}^{n}}]
}{\beta}.
\]
where all the expectations are w.r.t. the data distribution $P$. We can simplify the total variation distance further as follows:
\[
\E[\text{TV}\paran{P(\cdot|X_{1}^{n}),Q(\cdot|X_{1}^{n}}] \leq n \E[\text{TV}\paran{P(\cdot|X),Q(\cdot|X)}].
\]
We can now characterize the probability of the confidence set using the central limit theorem in a similar way:
\begin{align*}
\P\paran{ {Q(Y_{1}^{n}|X_{1}^{n})} \leq \beta  }   & \leq Q\paran{
\frac{-1}{\sigma\sqrt{n}}\log\beta +\frac{\sqrt{n}\mu}{\sigma}
}+ \frac{\rho}{\sqrt{n}\sigma^3}.
\end{align*}
The only difference is that the moments are computed for $\log Q(Y_i|X_i)$. We will come back to their computation later. First see that $\beta$ can be chosen in a way to guarantee the confidence level we are interested in:
\[
\beta \leq \exp\paran{n\mu - 
\sigma \sqrt{n} Q^{-1}(\alpha) + \frac{1}{2\pi}\frac{\rho}{\sigma^2},
}
\]
which we use to bound the expected set size
\[
\E[|\Gamma_Q^{\alpha}(X_{1}^{n})|] \leq 
 \paran{1 + n\E[\text{TV}\paran{P(\cdot|X),Q(\cdot|X)}]
} \exp\paran{-n\mu + 
\sigma \sqrt{n} Q^{-1}(\alpha) - \frac{1}{2\pi}\frac{\rho}{\sigma^2}.
}
\]
As last step we compute the moments as follows:
\begin{align*}
    \mu & = \E[\log Q(Y|X)] = - H(Y|X) - \E[D(P(\cdot|X)\Vert Q(\cdot|X))]\\
    \sigma & =\paran{\mathrm{Var}\paran{\log Q(Y|X)}}^{1/2}\\
    \rho &=\E\paran{\card{\log Q(Y|X) -\mu}^3}.
\end{align*}
\end{proof}

The key penalty for the distribution mismatch is the expected KL-divergence term $\E[D(P(\cdot|X)\Vert Q(\cdot|X))]$. This result also shows that one can directly try to approximate the conditional distribution if the error can be suitably controlled. For example, in \citet{sadinle2019least}, the authors used k-Nearest Neighbors, local polynomial estimator and Regularized multinomial logistic regression to approximate the conditional distributions.

\section{Algorithmic Details of Transductive Split Conformal Prediction}
\label{supp:algo_tscp}
In this section, we present the details of the transductive split conformal prediction algorithm using dynamic programming. The algorithm is presented in Algorithm \ref{alg:tcp_split_dp}. The underlying problem is a tree search. It starts with a dummy node (corresponds to $U_{n+1}$ in the algorithm. Each node at level $i-1$ branches out to $|\gY|$ node with the edge weight $f(x_i)[y]$ for $y\in\gY$. We are searching for the paths with the product of weights above a given value. The standard solution to such problems is dynamic programming where the branches are pruned as soon as the search constraint cannot be satisfied. The pruning condition is specified in Algorithm \ref{alg:tcp_split_dp}.

For a given product-threshold $\tau=1-q_{\alpha}$, the naive implementation of TSCP requires searching over all label tuples, namely $|\gY|^n$, where $n$ is the number of test samples in the transductive prediction. However, with dynamic programming, we can prune the paths that can never pass the threshold.  The set-size DP processes $n$ positions. At each step $j$ there are $P_{j-1} = \prod_{k < j} M_k$ live prefix log-products (where $M_k$ is the average branching factor at depth $k$), and the cost at that step is $\mathcal{O}(|\gY| \cdot P_{j-1})$, yielding a telescoping total of $\mathcal{O}\left(|\gY| \prod_{j=0}^{n-1} M_j\right) = \mathcal{O}(|\gY| \cdot S)$, where $S$ is the final set size. Therefore, the complexity can be much smaller for a classifier with spiked probabilities.

\begin{algorithm}[t]
\caption{Transductive Split Conformal Prediction with DP Prediction Step}
\label{alg:tcp_split_dp}
\begin{algorithmic}[1]
\Require Calibration data $\mathcal{D}_{\text{cal}}=\{( (x^{i}_1, y^{i}_1),\dots,(x^{i}_n, y^{i}_n) )\}_{i=1}^{N_{\text{cal}}}$, miscoverage level $\alpha \in (0,1)$
\Require A fixed predictive model $f$ returning class scores $f(x)[y]\in[0,1]$ for $y\in\mathcal{Y}$

\State Compute calibration nonconformity scores:
\[
s_i \;=\; 1 - \prod_{j=1}^n f(x^i_j)[y^i_j], 
\qquad i=1,\dots,N_{\text{cal}}.
\]

\State Compute the empirical $(1-\alpha)$ quantile (with finite-sample correction):
\[
q_{\alpha} = \text{Quantile}_{\left(1 - \alpha\right)\left(1 + \frac{1}{N_{\text{cal}}}\right)} \Big( \{s_i\}_{i=1}^{N_{\text{cal}}} \Big).
\]
\State Set the product-threshold $\tau \gets 1 - q_\alpha$.

\Function{Predict}{$x_1,\dots,x_n$}
    \State Precompute per-position class scores $p_j(y) \gets f(x_j)[y]$ for all $j\in[n], y\in\mathcal{Y}$.
    \State Precompute suffix upper bounds (for pruning):
    \[
    M_j \gets \max_{y\in\mathcal{Y}} p_j(y), 
    \qquad
    U_{n+1}\gets 1,\quad
    U_j \gets M_j \cdot U_{j+1}\ \text{ for } j=n,\dots,1.
    \]
    \Comment{$U_{j}$ upper-bounds the best achievable product from positions $j,\dots,n$}

    \State Initialize DP frontier (prefix products):
    \[
    S_0 \gets \{(\emptyset, 1)\},
    \]
    where each element of $S_j$ is a pair $(y_{1:j}, \pi)$ storing a prefix label tuple and its product $\pi=\prod_{t=1}^j p_t(y_t)$.

    \For{$j=1$ to $n$}
        \State $S_j \gets \emptyset$
        \ForAll{$(y_{1:j-1}, \pi)\in S_{j-1}$}
            \ForAll{$y\in\mathcal{Y}$}
                \State $\pi' \gets \pi \cdot p_j(y)$
                \If{$\pi' \cdot U_{j+1} \ge \tau$}
                    \Comment{prune prefixes that can never reach $\tau$}
                    \State Add $( (y_{1:j-1},y), \pi')$ to $S_j$
                \EndIf
            \EndFor
        \EndFor
    \EndFor

    \State Initialize prediction set $\Gamma(x_{1:n}) \gets \emptyset$
    \ForAll{$(y_{1:n}, \pi)\in S_n$}
        \If{$\pi \ge \tau$}
            \State $\Gamma(x_{1:n}) \gets \Gamma(x_{1:n}) \cup \{y_{1:n}\}$
        \EndIf
    \EndFor
    \State \Return $\Gamma(x_{1:n})$
\EndFunction
\end{algorithmic}
\end{algorithm}
\section{Comparison with prior works}
\label{supp:comparison}

\paragraph{Comparison with \citep{Correia2024-zq}.} The authors in  \citep{Correia2024-zq} provided information theoretic bounds on the efficiency of conformal prediction algorithms. The main bound on the expected set size is derived from Fano's inequality for variable size list decoding, given in Proposition C.7 of the paper:
\begin{equation*} 
        H(Y|X)\leq h_b(\alpha) + \alpha\log|\gY| + \E([\log|\gC(x)|]^+),
\end{equation*}
This is for a single test sample prediction. Since the bound holds for any space $\gY$ and $\gX$, we can use it for transductive confidence prediction by choosing the product space $\gY^n$ and $\gX^n$, which yields the following bound, assuming independent samples:
\begin{equation*} 
        n H(Y|X)\leq h_b(\alpha) + n\alpha\log|\gY| + n\E([\log|\gC(X)|]^+),
\end{equation*}
As $n\to\infty$, and using Jensen's inequality, we get:
\[
H(Y|X)\leq \alpha\log|\gY| + \gamma^{-}_m.
\]
The result implies that if $\gamma^{-}_m<H(Y|X)$, then
\[
\alpha \geq \frac{H(Y|X) - \gamma^{-}_m}{\log|\gY|},
\]
which means that the value $\alpha$ cannot be made arbitrarily small (or confidence arbitrarily high). Our result, as stated in Theorem \ref{thm:fundamental_asymp_transductive} is stronger, as it says that in such case the confidence goes to zero, or $\alpha\to 0$. 

This is analogous to the results in information theory  about weak and strong converses for Shannon capacity. The weak converse is proven using Fano's inequality and states that the rates above the capacity cannot have zero error. The strong converse states that the error goes to one. Fano's inequality is known to be loose in certain scenarios, which motivated many works on more efficient and tighter bounds in information theory (see \citep{polyanskiy_channel_2010} and references therein. 

\section{Transductive conformal prediction: first formulation} 
\label{supp:tcp_vovk}

As shown in \citep{Vovk2013-transductive-pred}, it should be noted that transductive conformal predictors are a class of transductive confidence predictors. Our theoretical bounds apply to all confidence predictors, which constitute a larger class. However, Theorem 3 in \citep{Vovk2013-transductive-pred} states, there is always a conformal predictor as good as a transductive one. Therefore, throughout the paper, we used mainly conformal predictors as our focus. However, the notion of nonconformity score, essential for transductive prediction, was not discussed in the paper. We review the confidence predictor using nonconformity score.

Transductive conformal predictor as in \citep{Vovk2013-transductive-pred} is defined using a \textit{transductive nonconformity score} $A: (\gX\times\gY)^* \times (\gX\times\gY)^* \to \R$ where $(\gX\times\gY)^*$ is the set of all finite sequence with elements $(X,Y)$, $X\in\gX, Y\in\gY$. $A(\zeta_1, \zeta_2)$ does not depend on the ordering of $\zeta_1$. The transductive conformal predictor for $A$, based on the labeled dataset given as $Z_{1}^{m}=((X_i,Y_i):i\in[m])$, compute the transductive nonconformity scores for each possible labels $\vv = (v_{m+1},\dots,v_{m+n})\in\gY^n$ of the test sequence $X_{m+1}^{m+n}=(X_{m+1},\dots,X_{m+n})$ as follows. Construct the labels $Y^{\vv}_{m+k}=v_{m+k}$ for $k\in[n]$, $Y^{\vv}_{i}=Y_{i}$ for $i\in[m]$. Consider the following definition:
\[
\mZ^{\vv}_S = ((X_i,Y^{\vv}_i):i\in S).
\]
Then, for each possible labels $\vv = (v_{m+1},\dots,v_{m+n})\in\gY^n$ and each ordered subset $S$ of $[m+n]$ with $n$ entries define:
\begin{equation}
    \xi^{\vv}_S := A(\mZ^{\vv}_{[m+n]\backslash S} \mZ^{\vv}_S ).
\end{equation}
and use to compute $p$-values:
\[
p(v_1,\dots,v_n) = \frac{|S: \xi^{\vv}_S\geq \xi^{\vv}_{\vv}|}{ (m+n)!/n!  }.
\]
These $p$-values can be used to construct the prediction sets as follows:
\[
\Gamma^{\alpha}(Z_{1}^{m},X_{m+1}^{m+n})=\{\vv = (v_{m+1},\dots,v_{m+n})\in\gY^n : p(v_1,\dots,v_n)\geq \alpha\}.
\]
Such construction comes with theoretical coverage guarantee that the predictor has the confidence at least $1-\alpha$ in the online mode (see Theorem 1 and Corollary 1 of \citep{Vovk2013-transductive-pred} for further discussions). 

As it can be seen from the above construction, computing all these $p$-values is computationally cumbersome. Therefore, one can try to construct transductive nonconformity measures from single nonconformity measures using another aggregator. Bonferroni predictors compute $p$-value for each test sample separately and the combine that using the Bonferroni equation:
\[
p:=\min(np_1,\dots,np_n,1),
\]
which amounts to the following modified prediction set:
\[
\Gamma^{\alpha}(Z_{1}^{m},X_{m+1}^{m+n})=\prod_{i=1}^n \left\{v_{m+i}\in\gY : p(v_{m+i})\geq \frac{\alpha}{n}\right\}.
\]
Bonferroni predictors have similar coverage guarantees to transductive conformal prediction (see Theorem 2 in \citep{Vovk2013-transductive-pred}). 

For our experiments, we use Bonferroni predictors for the $p$-values obtained from split conformal prediction (SCP). Although the method works based on computing $(1-\alpha)$-quantile, there is a 1-1 mapping to a $p$-value:  
\begin{align*}
     s \leq \quant(1-\alpha; \{S_i\}_{i=1}^n \cup \{\infty\}) \iff \frac{1}{n}\sum_{i=1}^n \mathbf{1}(S_i\geq s) > \alpha.
\end{align*}
In other words, the term $\frac{1}{n}\sum_{i=1}^n \mathbf{1}(S_i\geq s)$  is a $p$-value. Therefore, Bonferroni predictor for SCP can be obtained by running SCP per test sample using $1-\frac{\alpha}{n}$-quantile and then get the set product of predicted sets.

\subsection{On Optimal Confidence Predictors}
\label{app:optimal_conf_predictor}
The prior works considered the efficiency-confidence trade-off in context of confidence prediction \citep{Lei2014-zk,Lei2014-ss,Lei2013-ao,Lei2015-eb,sadinle2019least}.  In \citep{Lei2014-zk}, the case of binary classification is discussed where the prediction sets are constructed based on the thresholding of the conditional probability. Using Neyman-Pearson lemma, it is shown that such prediction sets achieve optimal efficiency. When the conditional probability is not given, its empirical version is used, which is shown to asymptotically achieve the optimal confidence prediction. Their analysis excludes the empty prediction sets. The optimal classifier of multi-class classifiers was discussed in \citep{sadinle2019least}. The solution is similarly based on the thresholding of the conditional probability. Our work is the general derivation of lower and upper bound on the optimal prediction sets in transductive setting, and is connected to information theoretic quantities as well.

\section{Supplementary Experimental Results}
\label{supp:experiments}
In this section, we present additional numerical results related to our theoretical bound. All experiments are with $N=10$ (corresponding to MNIST), and follows a similar setup presented in the main paper.

    \begin{figure}[ht!]
        \centering
        \includegraphics[width=0.7\textwidth]{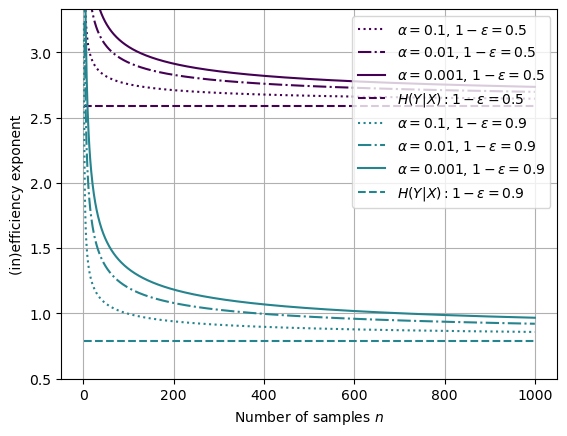} 
        \caption{The theoretical finite block length bounds for different noise levels, and different confidence $\alpha$}
        \label{fig:sub3}
    \end{figure}

    \begin{figure}[ht!]
        \centering
        \includegraphics[width=0.7\textwidth]{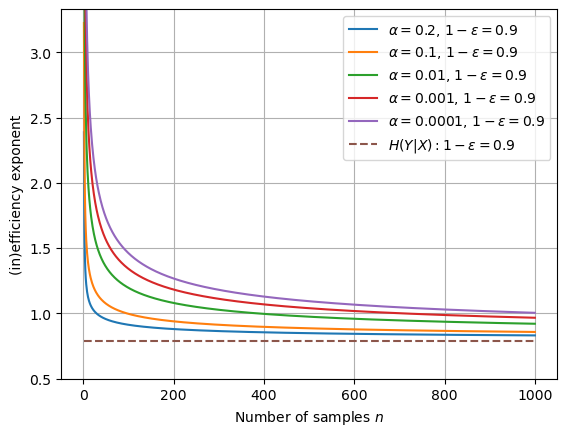} 
        \caption{The theoretical finite block length bounds for different confidence $\alpha$ in terms of $n$}
        \label{fig:sub4}
    \end{figure}

\paragraph{Simulating Theoretical Bounds.} Figures \ref{fig:sub3}, \ref{fig:sub4}, and \ref{fig:enter-label} are all based on simulating the theoretical bounds for noisy labels given in \eqref{eq:final_approximation}. We would like to observe a few trends, most of them intuitively expected. In Figure \ref{fig:sub3} and \ref{fig:sub4}, we plot the finite block-length bounds as a function of the number of test samples $n$ for different level of confidence. As the level of required confidence becomes more stringent, namely smaller $\alpha$, the inefficiency, given by the exponent of the expected set size, increases. Besides, the finite block length bound approaches slowly toward the asymptotic bound, $H(Y|X)$. Figure \ref{fig:sub3} plots the bounds for two different noise levels, which shows that changing noise level, i.e. intrinsic uncertainty, has a more drastic impact on the inefficiency.

\begin{figure}[ht!]
    \centering
    \includegraphics[width=0.9\linewidth]{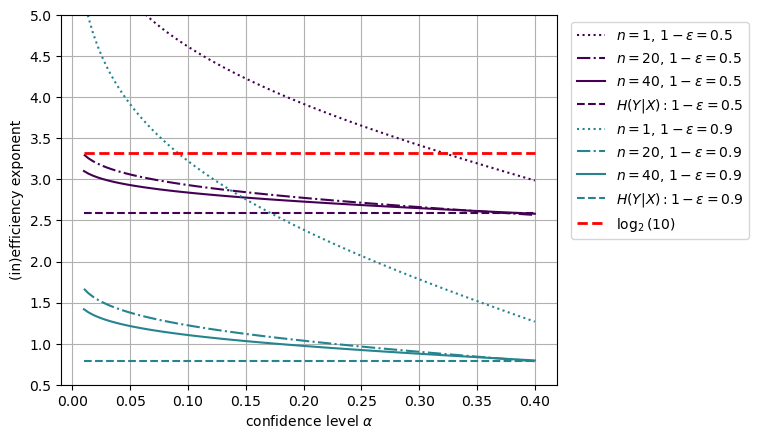}
    \caption{The theoretical finite block length bounds for different number of test samples $n$ in terms of confidence $\alpha$}
    \label{fig:enter-label}
\end{figure}

In Figure \ref{fig:enter-label}, we plot the bounds in terms of confidence levels. As we decrease the required confidence level by choosing larger $\alpha$, the inefficiency decreases as well. Similar to previous plots, choosing smaller $n$ increases the bound. 

If the full set $\gY$ is chosen every time as predicted set, it trivially included the correct label, but it yields the most inefficient prediction with the expected set size $\log|\gY|$. We plot it using the dashed red line, which 
shows $\log_2(10)$ for our experiment. The bound becomes vacuous, whenever it is above that line. There are a few reasons behind the vacuity of our bound. First of all, certain terms were ignored in the approximate bound, this includes $\log\Delta$, as well as a condition on $\alpha + \frac{\rho}{\sqrt{n}\sigma^3} + \Delta$ being within the interval $(0,1)$. 
We discuss these details in Remark \ref{remark:tightness_bound}.

\begin{figure*}[ht!]
    \centering
    \includegraphics[width=0.9\textwidth]{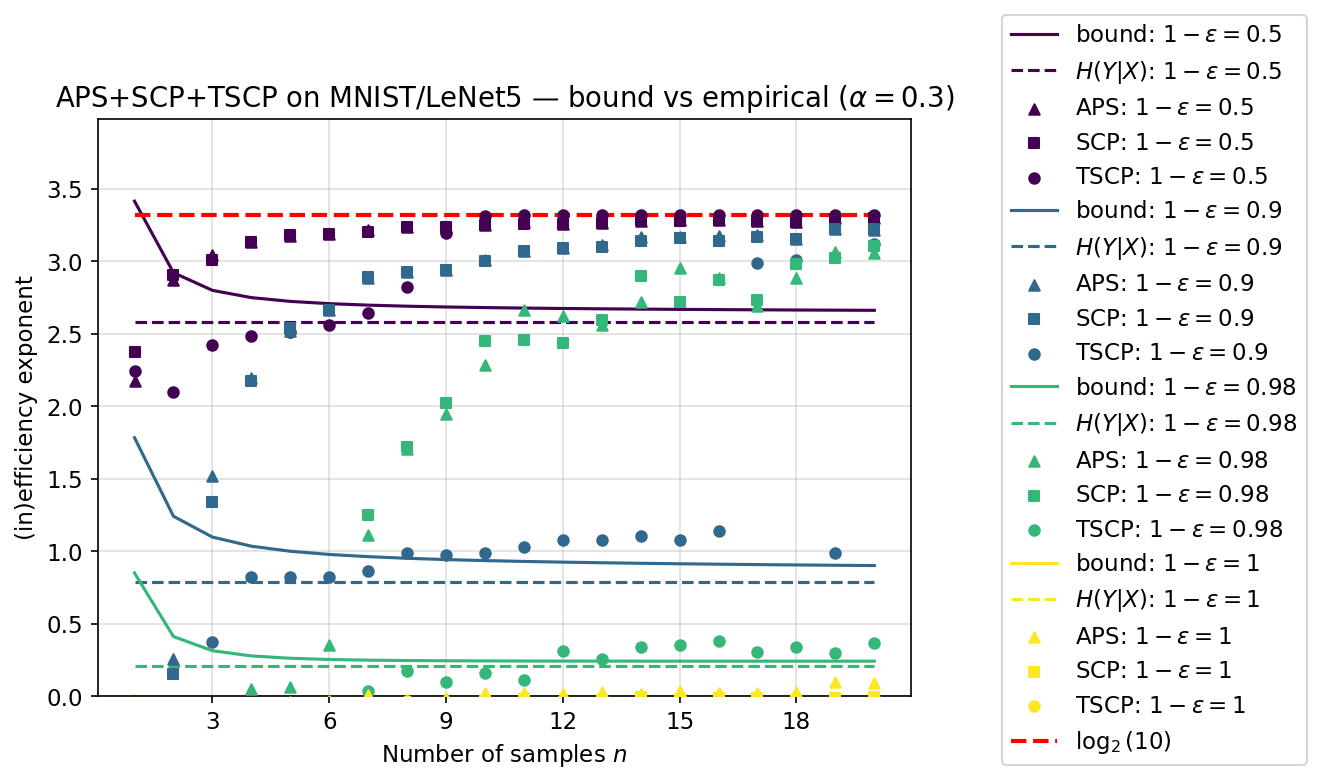}
    \caption{
    Comparison of the inefficiency exponent $\gamma_n$ for Bonferroni-corrected \ac{SCP}, \ac{APS}, and the proposed \ac{TSCP} for LeNet5 trained on noisy MNIST, with $\alpha=0.3$. 
    }
    \label{fig:mnist_tscp_w_all_0.3}
\vspace{-2mm}
\end{figure*}

\begin{figure*}[ht!]
    \centering
    \includegraphics[width=0.9\textwidth]{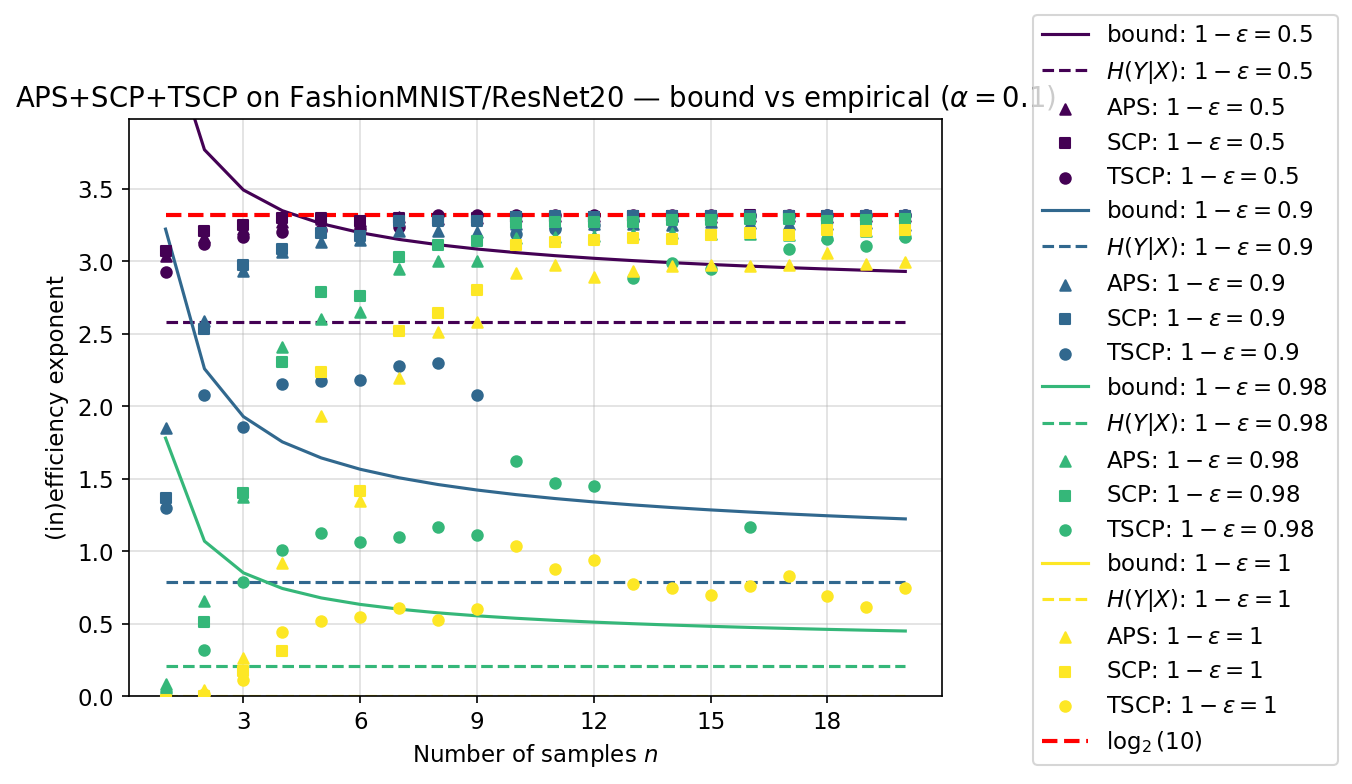}
    \caption{
    Comparison of the inefficiency exponent $\gamma_n$ for Bonferroni-corrected \ac{SCP}, \ac{APS}, and the proposed \ac{TSCP} for ResNet20 trained on noisy FashionMNIST, with $\alpha=0.1$. 
    }
    \label{fig:fmnist_tscp_w_all_0.1}
\vspace{-2mm}
\end{figure*}

\begin{figure*}[ht!]
    \centering
    \includegraphics[width=0.9\textwidth]{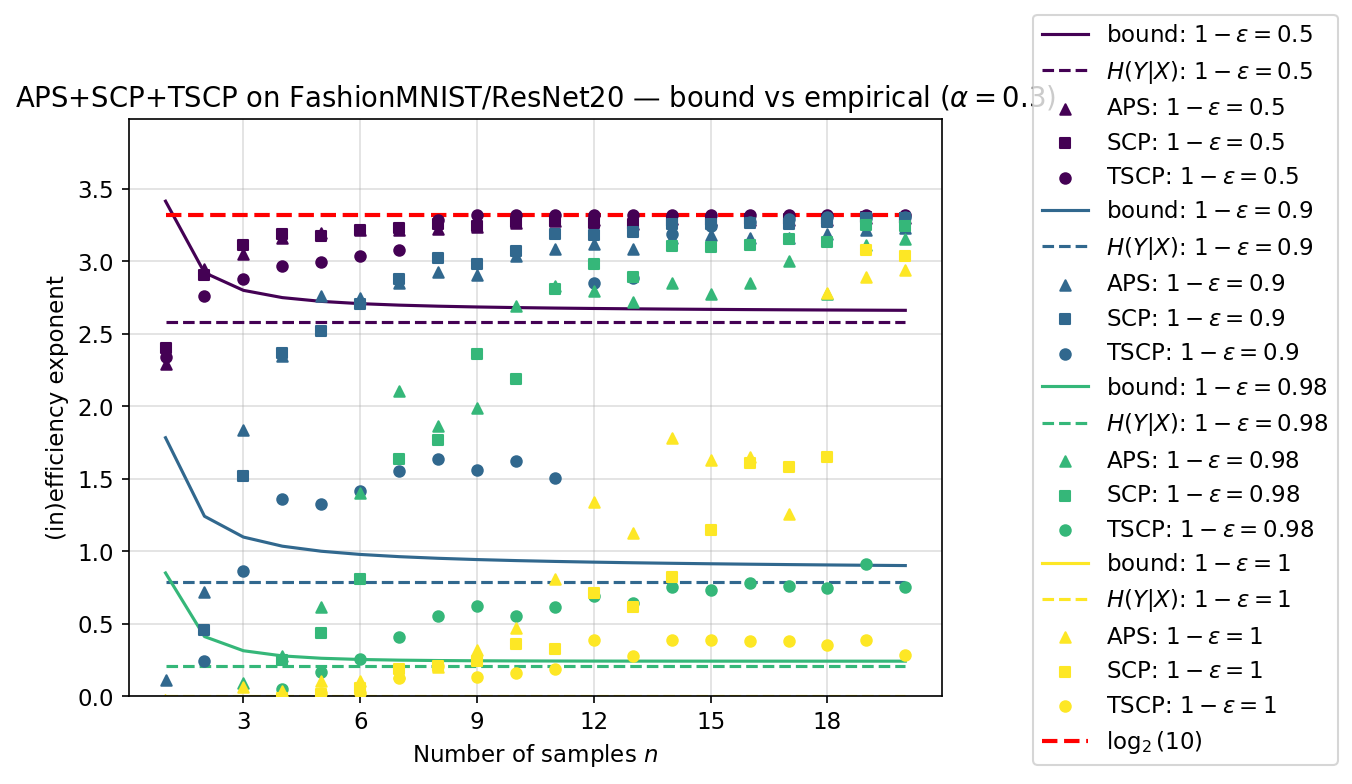}
    \caption{
    Comparison of the inefficiency exponent $\gamma_n$ for Bonferroni-corrected \ac{SCP}, \ac{APS}, and the proposed \ac{TSCP} for ResNet20 trained on noisy FashionMNIST, with $\alpha=0.3$. 
    }
    \label{fig:fmnist_tscp_w_all_0.3}
\vspace{-2mm}
\end{figure*}

\begin{figure*}[ht!]
    \centering
    \includegraphics[width=0.9\textwidth]{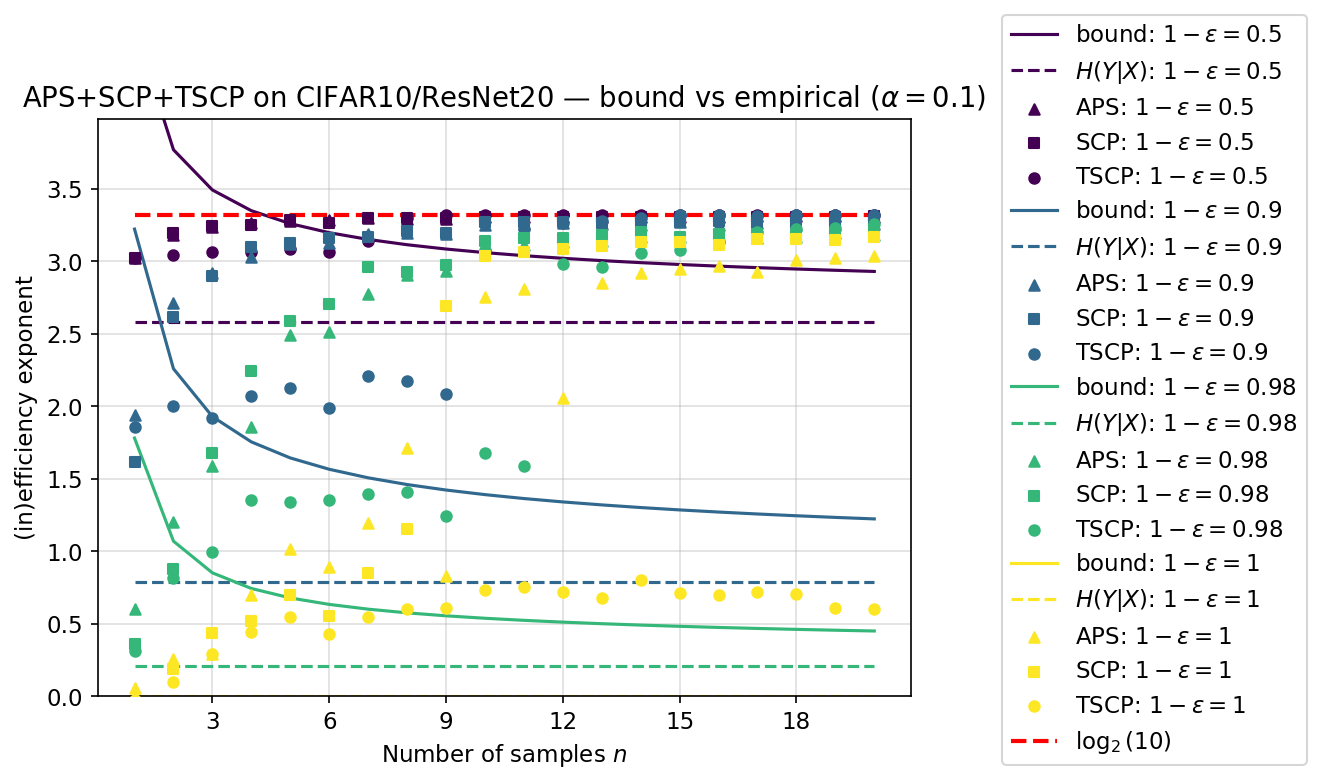}
    \caption{
    Comparison of the inefficiency exponent $\gamma_n$ for Bonferroni-corrected \ac{SCP}, \ac{APS}, and the proposed \ac{TSCP} for ResNet20 trained on noisy CIFAR10, with $\alpha=0.1$. 
    }
    \label{fig:cifar10_tscp_w_all_0.1}
\vspace{-2mm}
\end{figure*}

\begin{figure*}[ht!]
    \centering
    \includegraphics[width=0.9\textwidth]{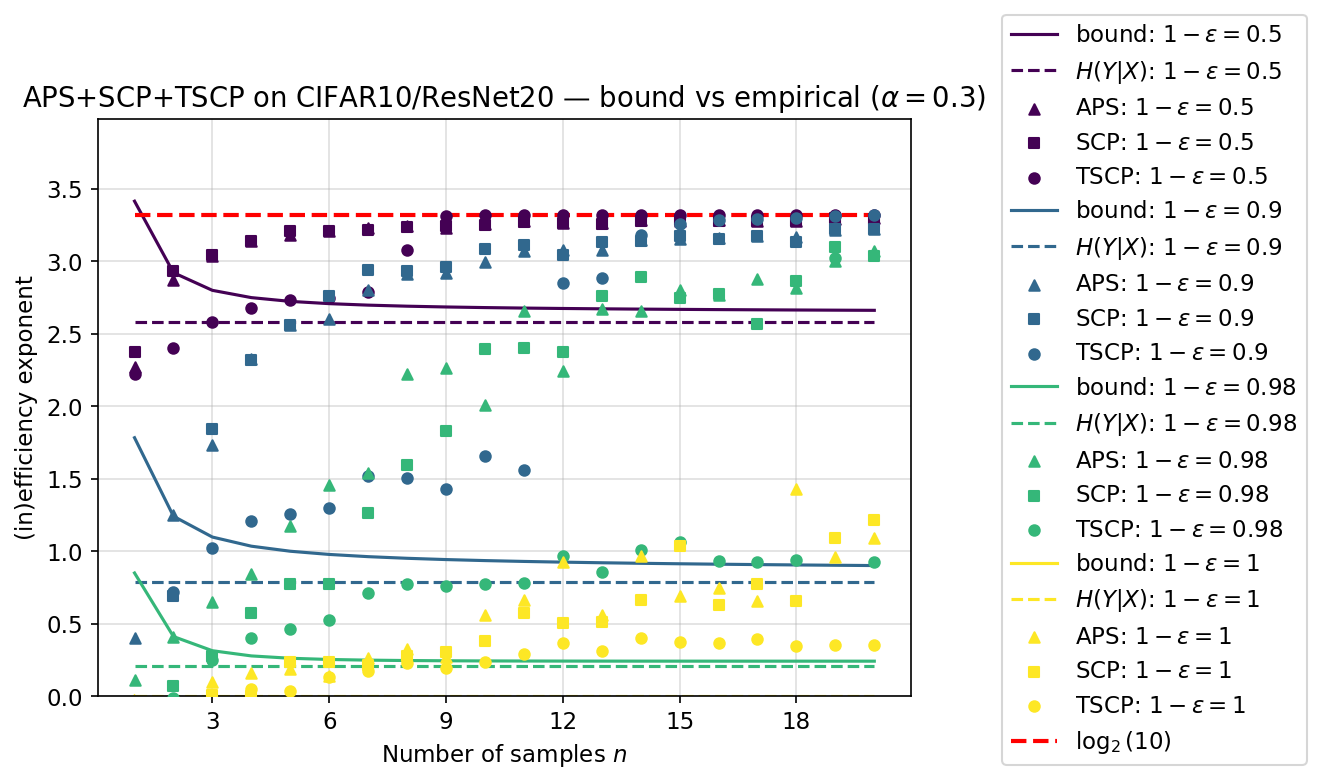}
    \caption{
    Comparison of the inefficiency exponent $\gamma_n$ for Bonferroni-corrected \ac{SCP}, \ac{APS}, and the proposed \ac{TSCP} for ResNet20 trained on noisy CIFAR10, with $\alpha=0.3$. 
    }
    \label{fig:cifar10_tscp_w_all_0.3}
\vspace{-2mm}
\end{figure*}

\begin{figure*}[ht!]
    \centering
    \includegraphics[width=0.9\textwidth]{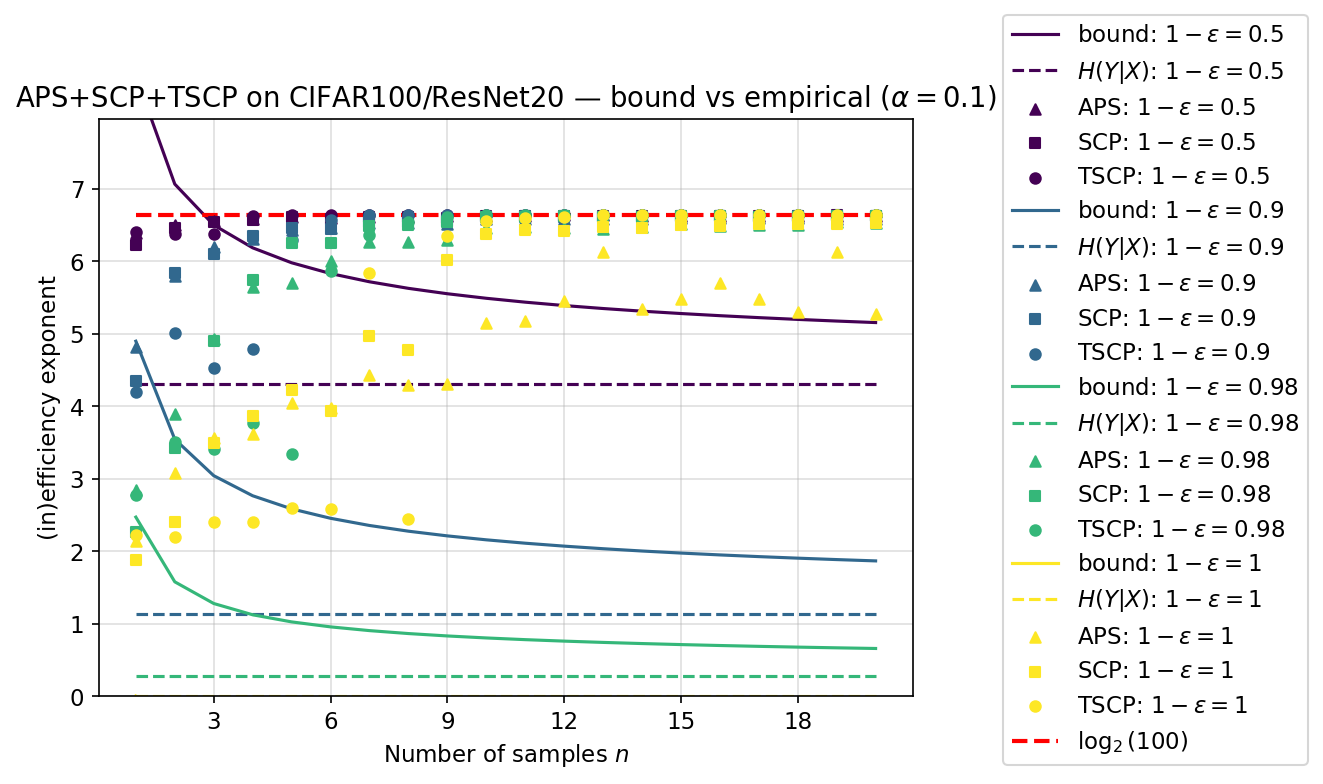}
    \caption{
    Comparison of the inefficiency exponent $\gamma_n$ for Bonferroni-corrected \ac{SCP}, \ac{APS}, and the proposed \ac{TSCP} for ResNet20 trained on noisy CIFAR100, with $\alpha=0.1$. 
    }
    \label{fig:cifar100_tscp_w_all_0.1}
\vspace{-2mm}
\end{figure*}

\begin{figure*}[ht!]
    \centering
    \includegraphics[width=0.9\textwidth]{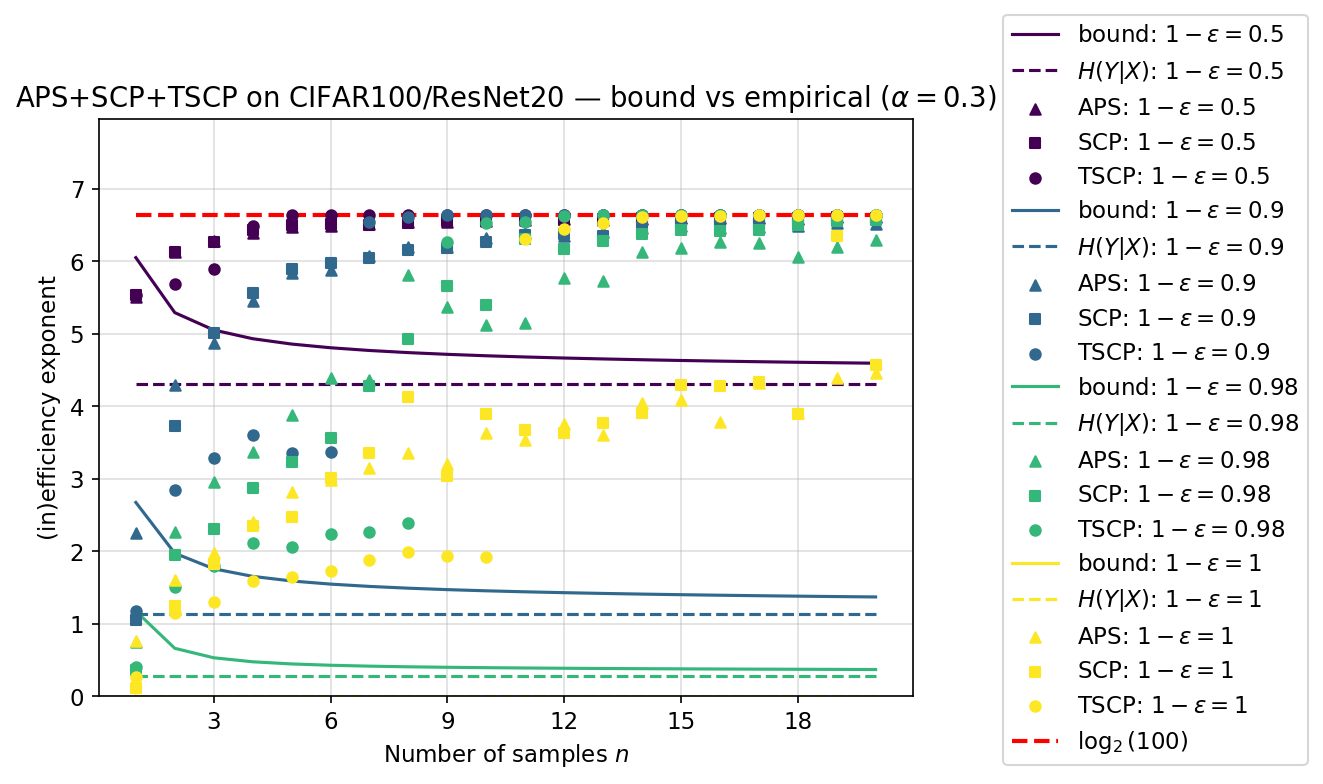}
    \caption{
    Comparison of the inefficiency exponent $\gamma_n$ for Bonferroni-corrected \ac{SCP}, \ac{APS}, and the proposed \ac{TSCP} for ResNet20 trained on noisy CIFAR10, with $\alpha=0.3$. 
    }
    \label{fig:cifar100_tscp_w_all_0.3}
\vspace{-2mm}
\end{figure*}

\paragraph{Numerical Comparisons with Transductive Methods.} 
In Figures \ref{fig:mnist_tscp_w_all_0.3},\ref{fig:fmnist_tscp_w_all_0.1}, \ref{fig:fmnist_tscp_w_all_0.3}, \ref{fig:cifar10_tscp_w_all_0.1}, \ref{fig:cifar10_tscp_w_all_0.3}, \ref{fig:cifar100_tscp_w_all_0.1} and \ref{fig:cifar100_tscp_w_all_0.3} we plotted inefficiency rate versus number of transductive samples for various models and the choice of confidence $\alpha=0.1,0.3$. Across the plots, a common trend is that our \ac{TSCP} provides a better inefficiency rate across the cases for $n<10$. As we discussed, this can be attributed to the smaller effective calibration size. We preserved the same calibration set size with Bonferroni-predicted methods to reflect the equal effort in data collection. However, this choice penalizes the transductive method for larger $n$. Another observation is that the more difficult datasets, for instance CIFAR100, the efficiency rate increases. This can be attributed to larger KL-divergence between the true conditional and the approximate one. 

\paragraph{Noiseless and Noisy Labels.} We have also included the numerical comparison of the efficiency rate for TSCP, SCP and APS in Tables \ref{tab:coverage_efficiency_0.3} for the confidence level $\alpha=0.3$, \ref{tab:coverage_efficiency_0.1_noiseless} for noiseless labels and \ref{tab:coverage_efficiency_0.1_noise0.9} for moderate noise $1-\epsilon=0.9$. These results paint a similar picture. \ac{TSCP} is consistently the more efficient approach while providing a similar coverage.
\begin{table*}[t]
\centering
\small
\caption{
Efficiency comparison of SCP, APS and TSCP for (MNIST: LeNet5), and (FashionMNIST, CIFAR10, CIFAR100: ResNet20) for different transductive sizes, $\alpha=0.3$, and $1-\epsilon=0.98$ as noise level. We compute the median and the expectation of the prediction-set size, and convert it to the efficiency rate using $\log_2(\cdot)/n$; coverage (Cov.) is the empirical joint coverage over the test batch.
}
\label{tab:coverage_efficiency_0.3}
\setlength{\tabcolsep}{4.5pt}
\begin{tabular}{lccc|ccc|ccc}
\toprule
\multirow{2}{*}{\textbf{Method}}
& \multicolumn{3}{c|}{$n=3$}
& \multicolumn{3}{c}{$n=6$} 
& \multicolumn{3}{c}{$n=9$} \\
\cmidrule(lr){2-4}\cmidrule(lr){5-7}\cmidrule(lr){8-10}
& $\gamma_{n,m}$ & Median & Cov.
& $\gamma_{n,m}$ & Median & Cov.
& $\gamma_{n,m}$ & Median & Cov. \\
\midrule
SCP (MNIST)
& -0.0615 & 0.0000 & 0.8400
& -0.0477 & 0.0000 & 0.7000
& 2.0215 & 0.0000 & 0.6600 \\

APS (MNIST)
& -0.0615 & 0.0000 & 0.6400
& 0.3496 & 0.0000 & 0.7000
& 1.9468 & 0.1469 & 0.7400 \\

\midrule
\textbf{TSCP (MNIST)}
& \textbf{-0.2146} &  {0.0000} & {0.6400}
& \textbf{-0.0597} &  {0.0000} & {0.7400}
& \textbf{0.0942} &  {0.0000} & {0.7400} \\
\bottomrule
\midrule
SCP (FMNIST)
& -0.0507 & 0.0000 & 0.7000
& 0.8103 &  0.1667 & 0.8000
& 2.3624 & 0.5094 & 0.8400 \\

APS (FMNIST)
& 0.0877 & 0.0000 & 0.6600
& 1.4049 & 0.1667 & 0.6600
& 1.9871 & 0.7505 & 0.7400 \\

\midrule
\textbf{TSCP (FMNIST)}
& \textbf{-0.0725} &  {0.0000} & {0.7600}
& \textbf{0.2593} &  {0.0000} & {0.6400}
& \textbf{0.6251} &  {0.3915} & {0.7200} \\
\bottomrule
\midrule
SCP (CIFAR10)
& 0.2773 & 0.0000 & 0.8000
& 0.7715 & 0.5975 & 0.6200
& 1.8287 & 0.9588 & 0.8200 \\

APS (CIFAR10)
& 0.6520 & 0.3333 & 0.8200
& 1.4566 & 0.7486 & 0.7400
& 2.2617 & 0.9624 & 0.7400 \\

\midrule
\textbf{TSCP (CIFAR10)}
& \textbf{0.2495} &  {0.0000} & {0.5800}
& \textbf{0.5267} &  {0.4845} & {0.7000}
& \textbf{0.7584} &  {0.6271} & {0.7400} \\
\bottomrule
\midrule
SCP (CIFAR100)
& 2.3043 & 1.6958 & 0.8000
& 3.5656 & 2.6004 & 0.7000
& {5.6643} &  {3.5609} & 0.8400 \\

APS (CIFAR100)
& 2.9601 & 1.6514 & 0.7200
& 4.3883 & 2.6794 & 0.8200
& \textbf{5.3667} &  {3.2094} & 0.7400 \\

\midrule
\textbf{TSCP (CIFAR100)}
& \textbf{1.7924} & {1.6436} & {0.7400}
& \textbf{2.2327} & {1.8604} & {0.7800}
&  {6.2748} & {1.9905} & {0.6600} \\
\bottomrule
\bottomrule
\end{tabular}

\vspace{0.5em}
\parbox{0.92\linewidth}{\footnotesize
\textbf{Notes.} The upper bound on the efficiency rate is $\log_2(\text{number-of-classes})$, which is 3.3219 for 10 classes, and 6.6438 for 100 classes.
}
\vspace{-5mm}
\end{table*}

\begin{table*}[t]
\centering
\small
\caption{
Efficiency comparison of SCP, APS and TSCP for (MNIST: LeNet5), and (FashionMNIST, CIFAR10, CIFAR100: ResNet20) for different transductive sizes, $\alpha=0.1$, and $1-\epsilon=1$ as noise level. We compute the median and the expectation of the prediction-set size, and convert it to the efficiency rate using $\log_2(\cdot)/n$; coverage (Cov.) is the empirical joint coverage over the test batch.
}
\label{tab:coverage_efficiency_0.1_noiseless}
\setlength{\tabcolsep}{4.5pt}
\begin{tabular}{lccc|ccc|ccc}
\toprule
\multirow{2}{*}{\textbf{Method}}
& \multicolumn{3}{c|}{$n=3$}
& \multicolumn{3}{c}{$n=6$} 
& \multicolumn{3}{c}{$n=9$} \\
\cmidrule(lr){2-4}\cmidrule(lr){5-7}\cmidrule(lr){8-10}
& $\gamma_{n,m}$ & Median & Cov.
& $\gamma_{n,m}$ & Median & Cov.
& $\gamma_{n,m}$ & Median & Cov. \\
\midrule
SCP (MNIST)
& \textbf{-0.0507} & 0.0000 & 0.9000
& -0.0253 & 0.0000 & 0.9000
& \textbf{-0.0134} & 0.0000 & 0.8800 \\

APS (MNIST)
& 0.0714 & 0.0000 & 0.8800
& 0.0094 & 0.0000 & 0.8800
& 0.0123 & 0.0000 & 0.9600 \\

\midrule
\textbf{TSCP (MNIST)}
& {-0.0298} &  {0.0000} & {0.9200}
& \textbf{-0.0363} &  {0.0000} & {0.8200}
& {-0.0099} &  {0.0000} & {0.8800} \\
\bottomrule
\midrule
SCP (FMNIST)
& 0.1754 & 0.0000 & 0.8400
& 1.4176 &  0.4308 & 0.9400
& 2.8065 & 0.5744 & 0.9800 \\

APS (FMNIST)
& 0.2608 & 0.0000 & 0.8400
& 1.3496 & 0.3333 & 0.9800
& 2.5844 & 0.4633 & 0.9400 \\

\midrule
\textbf{TSCP (FMNIST)}
& \textbf{0.1111} &  {0.0000} & {0.8400}
& \textbf{0.5478} &  {0.2642} & {0.9800}
& \textbf{0.6031} &  {0.4542} & {0.9000} \\
\bottomrule
\midrule
SCP (CIFAR10)
& 0.4368 & 0.0000 & 0.8800
& 0.5546 & 0.4308 & 0.8400
& 2.6948 & 0.4172 & 0.8800 \\

APS (CIFAR10)
& \textbf{0.2880} & 0.0000 & 0.9600
& 0.8913 & 0.3333 & 0.9800
& 0.8312 & 0.5094 & 0.9200 \\

\midrule
\textbf{TSCP (CIFAR10)}
& {0.2932} &  {0.0000} & {0.9400}
& \textbf{0.4316} &  {0.3333} & {0.8800}
& \textbf{0.6068} &  {0.4542} & {0.9200} \\
\bottomrule
\midrule
SCP (CIFAR100)
& 3.4899 & 2.4144 & 0.9200
& 3.9318 & 2.8739 & 0.8600
& {6.0168} &  {3.4213} & 0.8800 \\

APS (CIFAR100)
& 3.5698 & 2.4436 & 0.9400
& 3.9817 & 2.7293 & 0.9000
& \textbf{4.3061} &  {3.4631} & 0.9600 \\

\midrule
\textbf{TSCP (CIFAR100)}
& \textbf{2.4020} & {1.9271} & {0.9000}
& \textbf{2.5856} & {2.2654} & {0.8800}
&  {6.3501} & {2.1697} & {0.8600} \\
\bottomrule
\bottomrule
\end{tabular}

\vspace{0.5em}
\parbox{0.92\linewidth}{\footnotesize
\textbf{Notes.} The upper bound on the efficiency rate is $\log_2(\text{number-of-classes})$, which is 3.3219 for 10 classes, and 6.6438 for 100 classes.
}
\vspace{-5mm}
\end{table*}

\begin{table*}[t]
\centering
\small
\caption{
Efficiency comparison of SCP, APS and TSCP for (MNIST: LeNet5), and (FashionMNIST, CIFAR10, CIFAR100: ResNet20) for different transductive sizes, $\alpha=0.1$, and $1-\epsilon=0.9$ as noise level. We compute the median and the expectation of the prediction-set size, and convert it to the efficiency rate using $\log_2(\cdot)/n$; coverage (Cov.) is the empirical joint coverage over the test batch.
}
\label{tab:coverage_efficiency_0.1_noise0.9}
\setlength{\tabcolsep}{4.5pt}
\begin{tabular}{lccc|ccc|ccc}
\toprule
\multirow{2}{*}{\textbf{Method}}
& \multicolumn{3}{c|}{$n=3$}
& \multicolumn{3}{c}{$n=6$} 
& \multicolumn{3}{c}{$n=9$} \\
\cmidrule(lr){2-4}\cmidrule(lr){5-7}\cmidrule(lr){8-10}
& $\gamma_{n,m}$ & Median & Cov.
& $\gamma_{n,m}$ & Median & Cov.
& $\gamma_{n,m}$ & Median & Cov. \\
\midrule
SCP (MNIST)
& 2.8162 & 2.7429 & 0.9000
& 3.1161 & 3.0860 & 0.8800
& 3.2045 & 3.1923 & 1.0000 \\

APS (MNIST)
& 2.8314 & 2.8313 & 0.8800
& 3.1037 & 3.1133 & 0.8800
& 3.1770 & 3.1669 & 0.8600 \\

\midrule
\textbf{TSCP (MNIST)}
& \textbf{1.4445} &  {1.4641} & {0.8400}
& \textbf{1.4540} &  {1.3801} & {0.9200}
& \textbf{1.4961} &  {1.2871} & {0.9000} \\
\bottomrule
\midrule
SCP (FMNIST)
& 2.9737 & 2.9053 & 0.9200
& 3.1786 &  3.1289 & 0.9000
& 3.2814 & 3.3219 & 0.9600 \\

APS (FMNIST)
& 2.9325 & 2.8765 & 0.9400
& 3.1481 & 3.1553 & 0.9400
& 3.2051 & 3.2108 & 0.9400 \\

\midrule
\textbf{TSCP (FMNIST)}
& \textbf{1.8570} &  {1.7740} & {0.8600}
& \textbf{2.1825} &  {1.8881} & {0.9200}
& \textbf{2.0811} &  {1.8127} & {0.8600} \\
\bottomrule
\midrule
SCP (CIFAR10)
& 2.9025 & 2.8239 & 0.9200
& 3.1617 & 3.1519 & 0.8800
& 3.1944 & 3.1618 & 0.8400 \\

APS (CIFAR10)
& 2.9224 & 2.9116 & 0.9200
& 3.1277 & 3.1192 & 0.9600
& 3.1917 & 3.1828 & 0.9000 \\

\midrule
\textbf{TSCP (CIFAR10)}
& \textbf{1.9206} &  {1.6591} & {0.8600}
& \textbf{1.9853} &  {1.8361} & {0.9000}
& \textbf{2.0876} &  {1.8779} & {0.9200} \\
\bottomrule
\midrule
SCP (CIFAR100)
& 6.1060 & 5.9406 & 0.9000
& \textbf{6.4524} & 6.4242 & 0.8600
& {6.5564} &  {6.5492} & 0.9400 \\

APS (CIFAR100)
& 6.2001 & 6.1049 & 0.9600
& 6.4607 & 6.4597 & 0.8400
& \textbf{6.5126} &  {6.5079} & 0.9400 \\

\midrule
\textbf{TSCP (CIFAR100)}
& \textbf{4.5361} & {4.3462} & {0.8600}
& {6.5779} & {6.6439} & {0.8800}
&  {6.6439} & {6.6439} & {0.9200} \\
\bottomrule
\bottomrule
\end{tabular}

\vspace{0.5em}
\parbox{0.92\linewidth}{\footnotesize
\textbf{Notes.} The upper bound on the efficiency rate is $\log_2(\text{number-of-classes})$, which is 3.3219 for 10 classes, and 6.6438 for 100 classes.
}
\vspace{-5mm}
\end{table*}


\end{document}